\documentclass{article}

\usepackage{microtype}
\usepackage{graphicx}
\usepackage{subcaption}
\usepackage{booktabs}
\usepackage{array}
\usepackage{hyperref}

\usepackage[preprint]{icml2026}
\hypersetup{pdftitle={Similar Models, Broken Odds: Where Luce's Axiom Fails in Chatbot Arena},pdfauthor={Soovadeep Bakshi, Apurba Bose, Aishwarya Bose, Sandip Nandi},pdfsubject={}}

\usepackage{amsmath}
\usepackage{amssymb}
\usepackage{mathtools}
\usepackage{amsthm}
\usepackage[capitalize,noabbrev]{cleveref}

\theoremstyle{plain}
\newtheorem{proposition}{Proposition}
\theoremstyle{definition}

\newcommand{\logit}{\operatorname{logit}}
\newcommand{\x}{\ensuremath{\times}}

\icmltitlerunning{Similar Models, Broken Odds}

\begin{document}

\twocolumn[
  \icmltitle{Similar Models, Broken Odds:\\
    Where Luce's Axiom Fails in Chatbot Arena}

  \begin{icmlauthorlist}
    \icmlauthor{Soovadeep Bakshi}{ind}
    \icmlauthor{Apurba Bose}{ind}
    \icmlauthor{Aishwarya Bose}{ind}
    \icmlauthor{Sandip Nandi}{ind}
  \end{icmlauthorlist}
  \icmlaffiliation{ind}{Independent Researcher}
  \icmlcorrespondingauthor{Soovadeep Bakshi}{soovadeep.bakshi@utexas.edu}

  \icmlkeywords{Bradley-Terry, Luce's axiom, preference learning, LLM evaluation, Chatbot Arena}

  \vskip 0.3in
]

\printAffiliationsAndNotice{}

\begin{abstract}
Chatbot Arena gives each language model one Bradley--Terry score, which fixes how
often any model should beat any other. Choice theory predicts where this fails:
near-copies are compared on the few details in which they differ, so votes between
them are more decisive than their scores imply, exactly for the comparison a lab
cares about, a new version against the old one. On 1.8 million Arena votes, votes
between versions of one model are 1.37 times as decisive as Bradley--Terry
predicts, between sizes of one generation 1.26 times and between models of one lab
1.12 times; a release a year later points the same way, and the effect survives
tie, prompt, topic and timing checks. How alike two models' answers are worded and
formatted tracks the effect. Plain scores leave the odds between relatives about
9\% too close to even on held-out votes; a correction with four to six parameters
roughly halves this and predicts held-out votes better, mostly by modeling ties,
while the leaderboard order barely moves; frozen and carried to new releases, its
gains in head-to-head odds are too small for our test to establish. With language-model judges and
three-answer menus, a near-copy takes votes mostly from its twin when choices are
pooled over prompts (log-odds shift $+0.10$, 95\% interval $0.05$ to $0.16$), less
so within a prompt.
\end{abstract}

\section{Introduction}
\label{sec:intro}

Public leaderboards turn pairwise human votes into one number per model.
Chatbot Arena \citep{chiang2024chatbot}, the most influential of them, fits the
Bradley--Terry model \citep{bradley1952rank}: the log-odds that model $i$ beats
model $j$ is the difference of their scores. The same model sits under most
reward models used to train chatbots from human feedback
\citep{christiano2017deep,ouyang2022training,rafailov2023direct}.

One score per model is a strong promise. If $A$ usually beats $B$ and $B$
usually beats $C$, the scores already fix how often $A$ beats $C$, without ever
watching $A$ against $C$. This product rule is the pairwise face of Luce's choice
axiom \citep{luce1959individual}. Choice theory has known for six decades where
the axiom breaks. In Debreu's example \citep{debreu1960review}, two recordings of
one piece of music, later retold as the red bus and the blue bus, two nearly
identical options are compared on the few details in which they differ, so the
better one wins more often than its overall appeal implies, while either one
against a different kind of option is a matter of taste. Near-copies are told
apart more decisively than their scores suggest, and no single score per option
can capture that \citep{tversky1972elimination,thurstone1927law}.

Language models come in families of near-copies: dated versions of one model,
sizes of one generation, successive generations from one lab. These are exactly
the comparisons a lab cares about when it ships a new version, and exactly where
classical theory says one score per model should fail. We test whether it does.
For example, Llama 3 70B beats Llama 3 8B in 64.7\% of their 7,797 decisive
head-to-head votes, while their results against the same other models imply
60.3\%.

The question matters beyond leaderboards. Reward models for learning from human
feedback fit the same Bradley--Terry likelihood to preference pairs, and those
pairs mix two samples from one policy, the closest near-copies there are, with
answers from different models. Arena scores a model across prompts, whereas a
reward model scores an answer given its prompt, $r(x,y)$, so our evidence does
not show that a single temperature is misspecified there; we state it as a
hypothesis. Per-batch temperatures \citep{wu2024betadpo} and preference models
without one score per answer \citep{munos2024nash,swamy2024minimaximalist,sun2025rethinking}
relax this in different ways. And a lab deciding whether a new version beats the old one reads
exactly the sibling odds that one score per model gets wrong.

\paragraph{Contributions.}
\begin{itemize}
\item \textbf{A test tied to outside information.} Pairwise votes alone cannot
reveal which pairs are special \citep{cherapanamjeri2026learning,xing2026attention}.
We tie each pair's decisiveness to information fixed before any vote is seen,
model lineage and answer text, which turns a known identification limit into a
test with one to three parameters (\cref{sec:setup}).
\item \textbf{Siblings are judged more decisively.} On two Arena releases, related
models are separated more sharply than Bradley--Terry implies (on the large log in
proportion to how closely they are related; on the newer, smaller release more
weakly, but consistently), relative to unrelated pairs.
The effect survives alternative tie treatments, measured prompt-mix
adjustments and timing controls (\cref{sec:results}).
\item \textbf{Surface similarity tracks the effect.} How alike two answers are
worded and formatted predicts decisiveness and, in the newer release, leaves the
family labels little to add; similarity of meaning predicts less (\cref{sec:mechanism}).
\item \textbf{A small correction.} Four to six extra parameters (in lineage form,
a nested logit on two-answer votes, \cref{sec:setup}) beat plain Bradley--Terry and
models with hundreds of parameters on votes and matchups they never saw, mostly
by modeling ties, and improve the calibration of the odds between relatives.
Frozen and carried to new releases, gains in head-to-head odds are not
established; the leaderboard order stays intact (\cref{sec:fix}).
\item \textbf{A menu test.} With language-model judges and three-answer menus, a
near-copy takes its votes mostly from its twin when choices are pooled over
prompts. Within a prompt the shift is smaller and clear for two of four judges;
siblings' strengths move together across prompts, consistent with prompt
heterogeneity contributing to the pooled effect (\cref{sec:judges}).
\end{itemize}

\section{Related Work}
\label{sec:related}

\paragraph{Arena and its scoring.} Chatbot Arena \citep{chiang2024chatbot}
introduced crowd-sourced pairwise evaluation with a Bradley--Terry leaderboard and
showed that its votes agree with expert judgment; it assumes one score per model
and does not test that assumption. \citet{singh2025leaderboard} document uneven
sampling and private testing on Arena; our timing checks find no sign that
launch-period sampling drives the effect. \citet{haghtalab2026pluralistic} show that
Bradley--Terry aggregation breaks a stability property because voters differ.
Voter heterogeneity over attributes that siblings share, as in mixed logit
\citep{mcfadden2000mixed} or hidden context in preference learning
\citep{siththaranjan2024distributional}, is one route to our effect, which Arena
data cannot separate from similarity within a single vote. Elo ratings are least
reliable between models of similar strength \citep{boubdir2024elo}, and Arena's
style control adds length and formatting covariates to each comparison
\citep{li2024stylecontrol}; both shift scores, while our effect changes the scale of
a comparison.

\paragraph{Beyond one score per model.} \citet{ameli2025statistical} fit a
factored covariance between models on an earlier snapshot of the same Arena log.
A pair with small spread in that model is exactly a decisive pair, so the
covariance can express our effect, but it is never linked to model families or to
answer similarity; their held-out results are mixed across metrics. We reproduce
their fit with their code and compare held-out prediction with our own
implementation.
\citet{xing2026attention} use the same per-pair scale, motivated by limited
attention, and find that Arena votes violate the one-score rule overall.
Attention predicts that pairs which are hard to tell apart look \emph{less}
decisive; near-copies in fact look \emph{more} decisive.
\citet{cherapanamjeri2026learning} prove that pairwise data cannot identify the
full correlation structure between options; our test uses only the spread of
each pair's comparison, which becomes identifiable once it is tied to covariates
fixed in advance. General preference and blade-chest models
\citep{zhang2024beyond,chen2016modeling} capture intransitive win rates with
per-model embeddings; on these data they gain almost nothing out of sample.

\paragraph{Choice theory.} Luce's axiom \citep{luce1959individual}, Debreu's
critique \citep{debreu1960review}, Thurstone's comparative judgment
\citep{thurstone1927law}, elimination by aspects \citep{tversky1972elimination},
the attraction effect \citep{huber1982adding} and the nested logit and its
generalizations \citep{mcfadden1974conditional,mcfadden1978modeling,train2009discrete}
are classical. That similar options are told apart more sharply in binary choice
is itself an old finding \citep{becker1963probabilities,tversky1969substitutability,rumelhart1971similarity,mellers1994similarity},
and decision field theory explains it with correlated noise
\citep{busemeyer1993decision,roe2001multialternative}. In machine learning,
context-dependent choice models learn departures from independence of irrelevant
alternatives from data \citep{benson2016relevance,ragain2016pairwise,seshadri2019discovering,tomlinson2021learning};
the salient-features model of \citet{bower2020preference}, in which a pair is
compared on the features where its items differ, is the closest in spirit; and
testing the axiom in general needs very large samples
\citep{seshadri2019fundamental}. Our contribution is to test the binary-choice
prediction on a large modern leaderboard, tie it to measured answer similarity,
and quantify the consequence for scoring.

\paragraph{Language models as judges.} Strong models agree with human preference
judgments at rates close to inter-human agreement \citep{zheng2023judging} but
favor particular positions and lengths \citep{wang2023fair,dubois2024length}.
We use judges not as a substitute for human votes but as choosers whose menu we
control, and we balance answer positions by design.

\section{Setup: Decisiveness and the Product Rule}
\label{sec:setup}

For models $i$ and $j$ with scores $s_i, s_j$, Bradley--Terry sets
\begin{equation}
\logit P(i \succ j) = s_i - s_j,
\end{equation}
so log-odds add along any chain:
$\logit P(i\succ \ell) = \logit P(i\succ j) + \logit P(j\succ \ell)$. This product rule
is what one score per model promises. Luce's axiom says more: the odds of $i$ over
$j$ do not depend on what else is on the menu. Arena shows two answers at a time,
so its data test the product rule; \cref{sec:judges} tests the axiom itself.

\paragraph{Decisiveness.} By \emph{decisiveness} we mean how steeply the win odds
of a pair rise with its score gap. We add one multiplier per kind of pair,
\begin{equation}
\label{eq:k}
\logit P(i\succ j) = k_{ij}\,(s_i - s_j),
\end{equation}
with $\log k_{ij} = \gamma_{\mathrm{rel}(i,j)}$ or $\log k_{ij} = \beta^{\top} x_{ij}$,
where $\mathrm{rel}(i,j)$ is the pair's relation (versions of one model line,
sizes within one generation, models from one lab, or unrelated, the
reference with $k=1$) and $x_{ij}$ holds measured answer similarity. $k_{ij}>1$
means the pair's votes are more lopsided than its score gap implies.
$k$ multiplies log-odds, not probabilities: at $k=1.37$, a 55\% prediction becomes
57\% and a 70\% one 76\%. We fit scores and $k$ jointly by maximum likelihood.

\begin{proposition}[Thurstone reading]
\label{prop:thurstone}
Let utilities be $U_i = s_i + \varepsilon_i$ with jointly normal noise, common
variance $\sigma^2$ and correlation $\rho_{ij}$. Then
$P(i\succ j) = \Phi\big((s_i-s_j)/(\sigma\sqrt{2(1-\rho_{ij})})\big)$, so relative to
pairs with correlation $\rho_0$ the decisiveness is
$k_{ij} = \sqrt{(1-\rho_0)/(1-\rho_{ij})}$. Unless $\rho$ is equal within a
triple $i,j,\ell$, the probit gaps $\Phi^{-1}P$ add along it only on a
measure-zero set of scores, so no single score per model fits all three pairs.
\end{proposition}
\begin{proof}
$U_i - U_j \sim N\big(s_i - s_j,\, 2\sigma^2(1-\rho_{ij})\big)$; take the ratio of
the standard deviations. Along a chain $i,j,\ell$, additivity of the scaled gaps is
a linear equation in $s$, which holds identically only if
$\sigma_{ij}=\sigma_{j\ell}=\sigma_{i\ell}$.
\end{proof}

The proposition concerns probit gaps; the logistic product rule is its analogue,
so correlations read off logistic multipliers below are approximate. We expect near-copies to share most of their idiosyncrasies, so
$\rho_{ij}$ should be high and the comparison rest on the few details that differ.
The ratio $(1-\rho_{ij})/(1-\rho_0)=1/k_{ij}^2$ needs no further assumption. If
unrelated pairs are uncorrelated ($\rho_0=0$), the multipliers we estimate below
(1.37, 1.26, 1.12) correspond to noise correlations of 0.47, 0.37 and 0.20. Models
that share training data and tuning pipelines are unlikely to be uncorrelated, so
these are lower bounds: with $\rho_0=0.3$ they become 0.63, 0.56 and 0.44.

\paragraph{Ties.} Arena voters may call a tie. Davidson's model
\citep{davidson1970extending} keeps them: with $d = k_{ij}(s_i-s_j)$,
$P(i \text{ wins}) \propto e^{d/2}$, $P(j \text{ wins}) \propto e^{-d/2}$ and
$P(\text{tie}) \propto \nu_{ij}$, where the tie propensity $\nu_{ij}$ also varies by
relation or similarity and with each model's own tie tendency. ``Relatives tie
more'' and ``relatives are called more decisively'' are then separate parameters;
a Thurstone model with a tie threshold \citep{glenn1960ties} would tie both to one
correlation, which we leave for later work.

\paragraph{Nested logit.} The classical model for this pattern is the nested
logit \citep{train2009discrete}, which groups options in nests with dissimilarity
$\lambda\in(0,1]$. When the menu holds only two options from one nest, the upper
choice is trivial and $P(i\succ j)=\big(1+e^{-(s_i-s_j)/\lambda}\big)^{-1}$: exactly
\cref{eq:k} with $k_{ij}=1/\lambda$ of the pair's smallest shared nest. Our lineage
model is therefore the nested logit with nests lab $\supset$ generation
$\supset$ model line, fitted on two-answer menus. Its consistency condition
$1\ge\lambda_{\mathrm{lab}}\ge\lambda_{\mathrm{gen}}\ge\lambda_{\mathrm{line}}$
becomes $1\le k_{\mathrm{lab}}\le k_{\mathrm{size}}\le k_{\mathrm{version}}$. A scale
driven by measured similarity is a nested logit only if similarity respects such
a nesting; otherwise it is a reduced-form pairwise model. What
two-answer votes cannot test is the nested logit's prediction for larger menus;
\cref{sec:judges} tries.

\paragraph{What pairwise data can identify.} A free $k_{ij}$ for every pair
cannot be separated from the scores without one parameter per pair, and the full
correlation structure cannot be learned from pairwise votes at all
\citep{xing2026attention,cherapanamjeri2026learning}. Tying $k_{ij}$ to lineage and
answer text, fixed before any vote is seen, gives a hypothesis with one to three
parameters, with unrelated pairs as the reference and unrelated pairs launched
together as a placebo.

\paragraph{Model-free check.} For each pair with at least three shared opponents
we compare the direct head-to-head log-odds $\ell_{ij}$ with the log-odds implied
through the shared opponents $c$, the inverse-variance weighted average of
$\ell_{ic}-\ell_{jc}$; the two use disjoint votes. Under
Bradley--Terry the slope of direct on indirect is 1; a slope above 1 within a
relation means that relation is judged more decisively than the chain predicts.

\section{Data and Similarity Measures}
\label{sec:data}

All data are public; no new votes or participants are involved.
\begin{itemize}
\item \textbf{Arena battle log} (\texttt{clean\_battle\_}\allowbreak\texttt{20240814\_public}):
1,799,991 anonymous votes on 129 models, April 2023 to August 2024, ties marked.
The main data set.
\item \textbf{Arena human-preference 140k} \citep{lmarena2025hp140k}: 135,634
votes on 53 newer models (April to July 2025) with both answers. The replication
and the main similarity data.
\item \textbf{Arena human-preference 100k} \citep{lmarena2024hp100k}: 106,134
votes with answers from June to August 2024, overlapping the end of the log.
Similarity on the older models.
\end{itemize}

\paragraph{Lineage.} Each model's lab, generation and line come from public names
and model cards. Every pair is tagged \emph{version} (one model at different dates
or modes, e.g.\ \texttt{gpt-4-0314}/\texttt{gpt-4-0613}), \emph{size} (one lab and
generation, different size or tier, e.g.\ Llama-3 8B/70B), \emph{lab} (one lab,
different generation) or \emph{unrelated}, before any vote is used. The 2024 test
has 17 version, 49 size and 119 same-lab pairs (\hyperref[app:pairs]{Appendix~\ref*{app:pairs}}). Thirteen
models' tags are judgment calls (app endpoints and moving aliases,
search-augmented variants, unreleased code names, tiers whose place in the family
is not public, code-specialized offshoots); \cref{sec:robust} drops them.

\paragraph{Answer similarity.} From the first assistant turn of both answers in
every battle we measure: \emph{wording}, the TF-IDF cosine; \emph{meaning}, the
cosine of sentence-encoder embeddings \citep{reimers2019sentence,wang2020minilm}
of the first 128 tokens and, for a sample of battles, of the whole answer;
\emph{style}, minus the distance between the two answers' formatting profiles
(counts of headings, bullet and numbered lines, bold spans, code blocks, table
rows, paragraphs and emoji, plus length, all log-scaled and standardized); and the
\emph{length gap}. Per pair we average over battles and standardize across
pairs. \emph{Excess} versions subtract an additive model with one effect per
model, removing each model's general typicalness and keeping only how alike this
pair is. We treat wording as the primary measure and the others as exploratory.

\section{Siblings Are Judged More Decisively}
\label{sec:results}

\begin{figure}[t]
\centering
\includegraphics[width=\columnwidth]{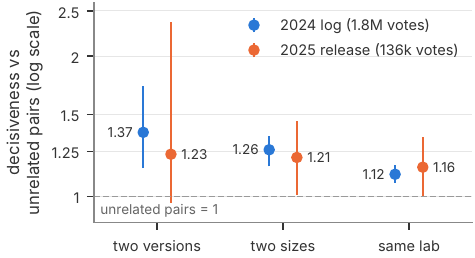}
\caption{Decisiveness of votes between related models relative to unrelated
pairs (\cref{eq:k}, ties dropped, pairs with at least 50 non-tied votes), with 95\% intervals from resampling pairs.
Both releases point the same way; the newer one has a thirteenth of the votes.
A multiplier of 1.37 means votes between two versions are as one-sided as
Bradley--Terry predicts for a score gap 37\% larger.}
\label{fig:multipliers}
\end{figure}

On the battle log, with ties dropped and pairs with at least 50 non-tied votes
(2,650 pairs, 1.14 million votes), the multipliers relative to unrelated pairs are
\x1.37 for version pairs (95\% interval 1.20 to 1.76 resampling whole model
families), \x1.26 for size pairs (1.14 to 1.34) and \x1.12 for same-lab pairs
(1.01 to 1.19); resampling pairs or voters gives narrower intervals, and dropping
any one family leaves 1.34 to 1.43, 1.22 to 1.29 and 1.04 to 1.15; the likelihood-ratio
statistic against Bradley--Terry is 177.6 on 3 degrees of freedom
(\cref{fig:multipliers}). The ordering follows relatedness, as expected if noise
correlation rises with relatedness (\cref{prop:thurstone}), and it satisfies the
nested-logit ordering. The 140k release, with different models a year later and a
thirteenth of the votes, points the same way. On its 981 pairs with at least 30
non-tied votes the multipliers are \x1.24, \x1.15 and \x1.17, jointly different
from Bradley--Terry (likelihood ratio 13.3 on 3 degrees of freedom, $p=0.004$);
with the 50-vote cut of \cref{fig:multipliers} they are \x1.23, \x1.21 and \x1.16,
and each interval alone reaches or nearly reaches 1 (version 0.97 to 2.37, size 1.01 to 1.45, lab
1.00 to 1.34). This is not a failed replication: the two releases are consistent
(heterogeneity $Q\le0.21$ on 1 degree of freedom for every relation), though the
newer one is too small to show they are equal, and pooled
they give \x1.35 (1.12 to 1.62), \x1.25 (1.17 to 1.34) and \x1.12 (1.07 to 1.17).
Selecting pairs by vote count leaves the likelihood valid, because Arena's
adaptive sampling depends only on votes already observed.

Votes are overdispersed relative to the binomial (Pearson $\chi^2$ per degree of
freedom 1.71 in 2024, 1.13 in 2025); dividing the likelihood ratios by it leaves
104 ($p<10^{-21}$) and 11.8 ($p=0.008$).

The model-free check agrees (\cref{fig:direct_indirect}). The slope of direct on
indirect log-odds is 0.99 for unrelated pairs, the reference, and 1.28, 1.23 and
1.10 for version, size and same-lab pairs. Relatives sit at smaller score gaps,
where the unrelated slope is slightly above 1 (1.08 in the closest fifth);
reweighted to each relation's gap distribution, unrelated pairs give 0.98 to 1.00,
so the excess slopes are $+0.28$ (95\% interval $+0.13$ to $+0.53$), $+0.24$
($+0.14$ to $+0.33$) and $+0.11$ ($+0.07$ to $+0.15$). One concrete case: plain scores give
\texttt{yi-large-preview} a 55\% chance against \texttt{yi-large}, but in 229
decisive head-to-head votes it wins 65\%.

Related models also tie more often: 44\% of votes for version pairs and 43\% for
size pairs, against 35\% for unrelated pairs, and the gap holds within three of the four
bands of score closeness that contain version pairs and within every band for
size pairs. Near-copies are either called a tie or called decisively.

\subsection{Robustness}
\label{sec:robust}

\begin{table}[t]
\caption{Decisiveness multipliers (version / size / same lab) on the 2024 log
under alternative choices. Sources: \texttt{ties}, \texttt{subsets}, \texttt{timing}, \texttt{clean}, \texttt{mixweight} and \texttt{multipliers.txt} in \texttt{results/2024/}.}
\label{tab:robust}
\centering\small
\setlength{\tabcolsep}{4pt}
\begin{tabular}{@{}lccc@{}}
\toprule
Check & version & size & lab \\
\midrule
Main (ties dropped, $\ge$50 votes) & 1.37 & 1.26 & 1.12 \\
All ties kept (Davidson) & 1.21 & 1.23 & 1.13 \\
Ties as half wins (Arena) & 1.19 & 1.10 & 1.06 \\
$\ge$200 votes per pair & 1.44 & 1.26 & 1.13 \\
English only & 1.33 & 1.17 & 1.11 \\
De-duplicated prompts & 1.37 & 1.25 & 1.10 \\
Code prompts & 1.32 & 1.12 & 1.08 \\
Non-code prompts & 1.35 & 1.29 & 1.12 \\
Weekly scores & 1.36 & 1.23 & 1.13 \\
Launch weeks dropped & 1.64 & 1.23 & 1.23 \\
Judgment-call pairs dropped & 1.44 & 1.27 & 1.12 \\
Prompt mix reweighted & 1.37 & 1.25 & 1.13 \\
\bottomrule
\end{tabular}
\end{table}

\paragraph{Ties.} With Davidson's model the multipliers are \x1.21, \x1.23 and
\x1.13 counting every tie and \x1.43, \x1.26 and \x1.14 counting plain ties only
(\cref{tab:robust}). Relatives are both more likely to tie (\x1.40, \x1.34 and
\x1.15 tie propensity) and more decisive when they do not. The version effect
depends on the tie rule; the size and lab effects are stable.

\paragraph{Prompts, labels and timing.} Prompt subsets, reweighting each pair's
votes to the prompt mix of unrelated pairs, vote thresholds, dropping
judgment-call lineage tags and fitting each calendar quarter all keep the
multipliers above 1 (\cref{tab:robust}, \hyperref[app:robust]{Appendix~\ref*{app:robust}}). Nor does launch-time
sampling explain them: with a separate score per model per week they are \x1.36,
\x1.23 and \x1.13, dropping each model's first four weeks makes the version
effect larger (\x1.64), and unrelated pairs launched within 14 days of each other
get \x0.99.

\paragraph{Prompt topics.} With prompts clustered into 20 topics, related models'
residual win rates move together across topics (excess correlation over
unrelated pairs $+0.34$, $+0.12$, $+0.15$ in the 140k release), yet scores that
vary by topic do not remove the multipliers; they rise (\x1.46, \x1.20, \x1.24;
100k release similar), so the effect persists after controlling for these 20
topics, though coarse topics cannot remove all prompt heterogeneity
(\hyperref[app:topics]{Appendix~\ref*{app:topics}}).

\subsection{Answer Similarity Tracks the Effect}
\label{sec:mechanism}

Relatives write alike: version pairs sit 2.4 to 3.0 standard deviations above the
average pair in wording similarity, size pairs 1.5 to 1.7 and same-lab pairs 0.8
to 0.9, and pairs with more similar wording tie more (correlation $+0.36$ to
$+0.43$). \Cref{tab:similarity} replaces the lineage labels in \cref{eq:k} by
similarity.

\begin{table}[t]
\caption{Decisiveness predicted by answer similarity: likelihood-ratio gain over
Bradley--Terry and multiplier per standard deviation. ``Adds'' is the gain from
adding the measure to the model in parentheses. Sources:
\texttt{results/*/similarity.txt}; 2024 meaning from \texttt{results/2024/meaning/}.}
\label{tab:similarity}
\centering\small
\begin{tabular}{@{}lcc@{}}
\toprule
 & 100k (2024) & 140k (2025) \\
 & 481 pairs & 981 pairs \\
\midrule
Lineage labels (3 df) & 15.5 & 13.3 \\
Wording only & 9.6, \x1.08 & 66.9, \x1.12 \\
Style only & 14.1, \x1.07 & 38.1, \x1.09 \\
Meaning, first 128 tokens & 2.2, \x1.03 & 26.8, \x1.07 \\
Meaning, whole answer & -- & 20.3, \x1.06 \\
Lineage adds (to wording) & 7.9 & 0.8 \\
Style adds (to wording + gap) & 7.1 & 1.1 \\
\bottomrule
\end{tabular}
\end{table}

In the 140k release, wording similarity alone predicts decisiveness (\x1.12 per
standard deviation, likelihood ratio 66.9 on one degree of freedom) and leaves
the family labels little to add (0.8), though that is not proof of full mediation.
Meaning similarity matters less and leaves more to the labels, and reading the whole answer is no better than reading its
opening; once wording is in, neither style nor whole-answer meaning adds anything.
Answers that differ a lot in length are judged less decisively (\x0.92 per
standard deviation alone). The wording effect survives controls for length gap and
score closeness (\x1.08) and keeping ties (\x1.09).

In the 100k release, style is the best single measure, ahead of wording, and adds
to it; list structure carries most of it (bullets add a likelihood ratio of 6.7
to wording, numbered lists 3.6, code blocks 3.8). The family labels still add
information on top of every text measure (likelihood ratio 6.0 on top of style).
The two releases thus disagree on whether surface similarity accounts for the
whole family effect: nearly in 2025, not in 2024. Sample size does not explain
this: cut to the 2024 fit's 50,000 non-tied votes, with similarity recomputed from
each subsample, wording still leaves the 2025 labels little to add (lineage adds a
median likelihood ratio of 1.2, and more than 7.9 in none of 50 draws), so the
difference lies between the releases, not their sizes. What
both agree on is the ranking of measures: surface similarity predicts decisiveness
better than similarity of meaning (in 2024 a likelihood ratio of 2.2 for meaning
of the opening against 9.6 for wording and 14.1 for style). These are associations. Shared lineage causes
similar answers, so similarity is better read as a mediator that our measures
capture in part than as an independent cause; pairs that share a prompt mix (for
example code) are also more similar for that reason, and we have not separated
the two. Among the measures, wording is primary; style, meaning and the
per-feature results are exploratory.

\section{Similarity-Aware Corrections and Their Limits}
\label{sec:fix}

All models keep ties with per-model tie tendencies; the fixes change both a
pair's win odds and its tie propensity: \emph{plain Bradley--Terry}; the
\emph{lineage fix}, \cref{eq:k} with one decisiveness and one tie multiplier per
relation (6 extra parameters); the \emph{similarity fix}, with both depending on
wording similarity, length gap and style (6); the factored-covariance model of
\citet{ameli2025statistical} (rank 3, 4 per model; it matches their
\texttt{leaderbot} fit); a rank-2 \emph{general preference} model
\citep{zhang2024beyond,chen2016modeling} (4 per model); and a ridged
\emph{per-pair scale} after \citet{xing2026attention}.

There are three held-out tests: a random fifth of the votes; a fifth of the model
pairs held out entirely, the test of \citet{ameli2025statistical} (half of each
held-out pair's battles supply answer text only); and later votes, training
before the 60th to 90th percentile of time and testing on the next tenth (5 seeds
or 4 cutoffs). Answer similarity never comes from a test battle. The rivals'
ridge is chosen on validation data; our fixes have no tuning. Gains are
differences in mean log loss; plain Bradley--Terry loses about 1,050 millinats
per vote, so a gain of 1 is about 0.1\%.

\begin{figure*}[t]
\centering
\includegraphics[width=\textwidth]{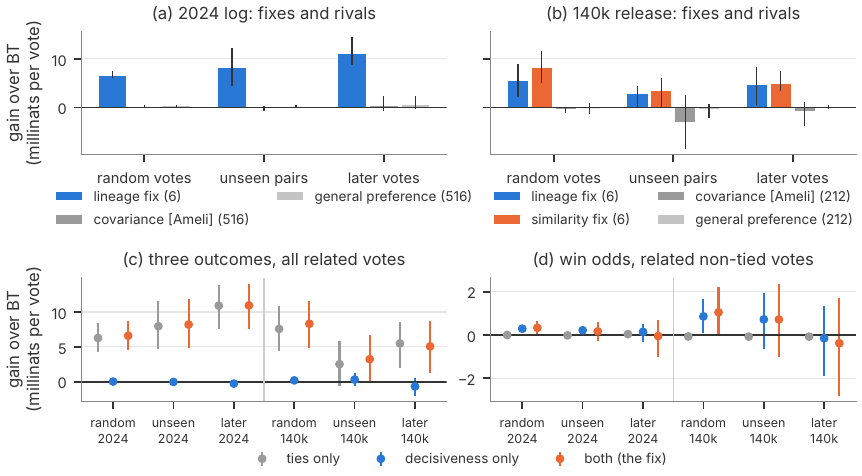}
\caption{Held-out gain over plain Bradley--Terry on votes between related models
(millinats per vote). (a, b) Three-outcome loss of the fixes and rivals (mean and
range over seeds or cutoffs; parameter counts in parentheses). (c, d) Each part of
the fix alone (2024: lineage; 140k: similarity), with 95\% intervals resampling
model pairs. All votes: \cref{tab:heldout}; all cells: \cref{tab:ablation}.}
\label{fig:heldout}
\end{figure*}

The small fixes beat every rival without outside information in every test of
both releases (\cref{fig:heldout}a--b, \cref{tab:heldout}). Every overall gain
of the fixes is above zero, and so is every gain between relatives except the
lineage fix's on unseen 2025 pairs ($+2.8$, $-0.7$ to $+6.2$). The flexible
models, which must learn from votes alone which pairs are special, gain at most
0.2 millinats per vote overall and lose on unseen pairs. Adding the six lineage
parameters to the covariance model matches the lineage fix on the 2024 log, so
the outside information carries the improvement. In the 140k release measured
similarity beats the lineage labels on every split; the specifications were
chosen on the same data, so this is not a confirmatory test of the measure.

\paragraph{Ties or win odds?} Switched on alone
(\cref{fig:heldout}c--d, \cref{tab:ablation}), the tie terms carry most of the
three-outcome gain ($+0.91$ of $+0.96$ overall on random 2024 votes). The
decisiveness term alone improves win odds between relatives on random and unseen
2024 votes ($+0.29$ and $+0.22$ millinats, intervals above zero), not detectably
on later votes, and moves the calibration slope on relatives' non-tied votes from
$1.09$ ($1.06$ to $1.14$, too close to even) to $1.04$ ($1.00$ to $1.08$).

\begin{figure}[t]
\centering
\includegraphics[width=\columnwidth]{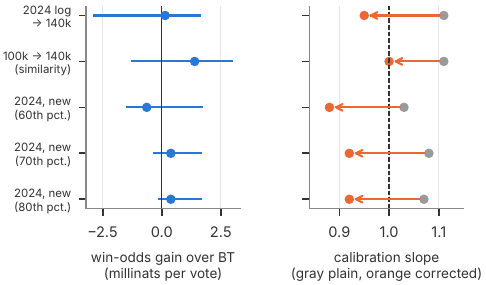}
\caption{Frozen transfer to target votes between related models. Left: win-odds
gain of the frozen correction over plain scores (95\% intervals resampling labs).
Right: calibration slope with plain scores (gray) and with the frozen correction
(orange); 1 is calibrated. Details: \cref{tab:transfer}.}
\label{fig:transfer}
\end{figure}

\paragraph{Frozen transfer.} We fixed a transfer test before running it
(\texttt{transfer.py}, committed first): learn the correction on earlier data,
freeze it, fit target scores on votes between unrelated models only, and predict
every related pair's votes. The targets were analyzed above, so the test is
retrospective. The primary outcome, the lineage fix's win-odds gain from the 2024
log to the 140k release, is $+0.14$ millinats ($-2.92$ to $+1.68$, resampling
labs); its three-outcome gain, $+9.7$ ($-0.5$ to $+19.3$), also includes zero, so
the test is inconclusive (\cref{fig:transfer}). Tie terms alone have three-outcome
intervals above zero in all four lineage settings, as does the full similarity fix
($+14.0$). Plain scores have slopes of 1.03 to 1.11; the frozen lineage multipliers'
slopes fall below 1 (0.88 to 0.95), while the frozen similarity fix moves the slope
from 1.11 to 1.00 and cuts relatives' win-rate error from 3.1 to 2.8 points; its
win-odds interval includes zero.

\paragraph{Why might the frozen multipliers overcorrect?} Post hoc diagnostics
(\hyperref[app:why]{Appendix~\ref*{app:why}}) argue against score-estimation noise (sibling-gap reliability 0.99
and 1.00) and suggest contributions from joint
estimation and changes between releases. On these scores, the best multiplier on
sibling votes is 1.10 for new 2024 releases (1.02 to 1.24 resampling labs) and
1.11 in the 140k release, whose lab interval (0.89 to 1.28) includes 1. In-sample
multipliers run a few percent above it even on the target itself, and the 2024
multiplier fell between periods (1.17 to 1.12); no such difference has
paired-bootstrap intervals excluding zero under both resampling schemes. The
attainable gain is small, as log loss is flat near even odds: at sibling gaps of
about 0.5 logits a 10\% steeper slope moves win rates by about one point, and
shared or relation-specific multipliers fitted to the target's sibling votes, on
fixed scores, gain only 0.40 and 0.61 millinats, too little for lab
intervals to resolve.

\subsection{What Changes on the Leaderboard}

On the battle log the lineage fix leaves the order almost unchanged: the rank
correlation with plain scores is 0.9999, the mean model moves 2.2 Elo points and
the largest move is 8; Arena-style ranks move by at most 3 places, and no sibling
pair swaps order, because about 89\% of votes are between unrelated models. In the 140k
release the similarity fix moves models by up to 21 Elo points, most for models
whose answers are unlike everyone else's.

\section{A Menu Test with Language-Model Judges}
\label{sec:judges}

Arena only ever shows two answers, so its votes cannot show whether adding a
near-copy of $A$ to a choice between $A$ and $B$ takes votes mostly from $A$, the
red-bus prediction of the nested logit. We control the menu with language-model
judges.

\paragraph{Design.} Eight inexpensive models answered 200 English Arena prompts
drawn at random: Llama 3.1 8B, 3.1 70B and 3.3 70B; Mistral Small 3.1 and 3.2;
Qwen3 14B and 32B; and GPT-4o mini. This gives two version pairs (Llama 3.1/3.3
70B, Mistral Small 3.1/3.2) and two size pairs (Llama 3.1 8B/70B, Qwen3 14B/32B).
Four judges from four labs, Claude Haiku 5.5, GPT-4.1 nano, Gemini 2.5 Flash-Lite
and DeepSeek V3.2, saw the prompt and two or three anonymous answers and replied
with one letter, sampled at temperature 1. Per prompt each judged all 28 pairs in
both orders; each sibling pair with four fixed outsiders $T$ in three cyclic
orders (sibling triads $\{S_1,S_2,T\}$); and two placebos with $S_2$ replaced by a
model $U$ from a third lab, one drawn at random and one matched to $S_2$'s
strength (chosen from the first two judges' scores before the last two ran; it
differs from the random one in 10 of 16 designs): 138,400 judgments in all.

\paragraph{Test.} Luce's axiom says the odds of $T$ over $S_1$ do not depend on the
third option. For each of the 16 designs,
\begin{equation}
\delta = \log\frac{P(T)}{P(S_1)}\Big|_{\{S_1,S_2,T\}} - \log\frac{P(T)}{P(S_1)}\Big|_{\{S_1,U,T\}},
\end{equation}
with shares pooled over the three orders. $\delta=0$ under the axiom; $\delta>0$
when the sibling takes its votes mostly from its twin, the similarity effect. Balancing orders does not remove position effects: a judge
that sometimes picks a slot regardless of content pulls shares toward a third,
which shifts $\delta$ whenever $U$ and $S_2$ differ in strength. The matched
placebo removes most of this without a model of the judge. We also fit, per judge
and to all its judgments, $P(i) = (1-\varepsilon)\,\mathrm{NL}_i +
\varepsilon/n$, a nested logit on utilities $\mu s_i + b_{\text{slot}}$ in which
related answers share a nest with dissimilarity $\lambda_{\text{rel}}$, plus a
uniform lapse $\varepsilon$; $\mu$ is 1 on pairs and free on triads, and $\lambda$
is estimated separately for two- and three-answer menus. On pairs the nest only
rescales the score gap by $k=1/\lambda$ (\cref{eq:k}); on triads $\lambda<1$ is
the red bus, given the judge's own scores, position habits and lapse.

\begin{figure}[t]
\centering
\includegraphics[width=\columnwidth]{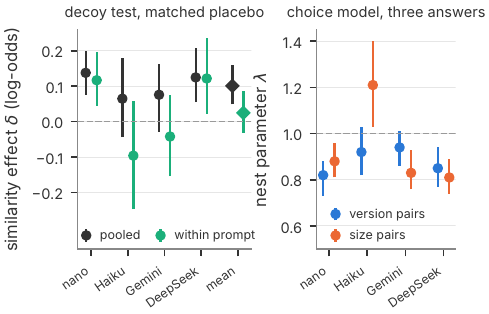}
\caption{Left: similarity effect $\delta$ per judge against the strength-matched
placebo, pooled over prompts (mean over 16 designs) and within prompts
(Mantel--Haenszel over design-by-prompt strata); diamonds: mean of the four judges.
Right: the choice model's nest parameter on three-answer menus; $\lambda=1$ is
Luce, $\lambda<1$ the red bus. Bars: 95\% intervals from resampling prompts.}
\label{fig:decoy}
\end{figure}

\paragraph{Results.} Against the matched placebo, a near-copy of $S_1$ in place of
an unrelated answer raises the odds of the outsider over $S_1$ for every judge
($\delta$ from $+0.07$ to $+0.14$, \cref{fig:decoy}). The
mean of the four is $+0.10$ ($+0.05$ to $+0.16$), and 51 of
64 designs are above zero; against the random placebo the mean is also $+0.10$
($+0.04$ to $+0.15$), so the shift is not an artifact of unequal strength. The
choice model agrees: the three-answer $\lambda$ is below 1 in seven of eight
judge-relation cells, with intervals excluding 1 in five (\cref{tab:judges}),
and the fitted lapse rates are small (0.00 to 0.07). A
$\lambda$ of 0.8 to 0.9 corresponds to $k$ of 1.1 to 1.25, the size of Arena's
multipliers. The exception is Claude Haiku 5.5 on size
pairs ($\lambda=1.21$, 1.03 to 1.40), where the larger sibling takes share from
the outsider instead. The two judges that played no part in choosing the placebo
give the same mean, $+0.10$ ($+0.02$ to $+0.17$; \hyperref[app:judges]{Appendix~\ref*{app:judges}}).

\paragraph{Pooled or within a prompt?} $\delta$ pools choices over prompts. If each
prompt's choice obeyed Luce but siblings were strong and weak on the same
prompts, pooled odds could still move with the third option. Each prompt's
answers are fixed and judged three times per menu, so we also compare odds within
prompts (Mantel--Haenszel over design-by-prompt strata). There the shift is
$+0.03$ ($-0.03$ to $+0.09$), positive for GPT-4.1 nano and DeepSeek V3.2 ($+0.12$
each, intervals above zero) but not the other two (\cref{fig:decoy}). In the
pairs, a model's win share on a prompt moves with its sibling's (correlation
across prompts 0.35 for version pairs, 0.27 for size pairs and 0.25 for the
same-lab pair, against 0.09 for unrelated pairs; all excesses exclude zero). The
attenuation after conditioning on prompts, together with this correlated
performance, is consistent with prompt heterogeneity contributing to the pooled
menu effect; it does not rule out a smaller within-prompt effect, which two
judges show. Nor do pairs and triads share one
nested logit: pair-fitted nest parameters, carried to triads, fit worse than
$\lambda=1$ for three of four judges, and Haiku's triad size value (1.21) lies
outside the range a random-utility nested logit allows.

\paragraph{Human ratings.} In PRISM \citep{kirk2024prism}, 1,396 people rated two
to four answers from 21 late-2023 models on one screen. Ratings of related answers
in a menu move together (residual correlation 0.50, 0.27 and 0.33 for versions,
sizes and same lab, against 0.14 for unrelated; all excesses exclude zero), as
correlated noise would. In choices the red-bus shift is clear only for Claude
2.0/2.1 ($+0.40$, $0.08$ to $0.69$; \hyperref[app:prism]{Appendix~\ref*{app:prism}}).

\section{Discussion and Limitations}
\label{sec:discussion}

Arena's ranks are robust: most votes are between unrelated models, and a
correction for relatives moves the order little. Its implied head-to-head
odds between siblings are too close to even, by about 9\% on the log-odds scale on
held-out 2024 votes and 3 to 11\% when scores come from unrelated votes alone;
these are the odds a lab reads when it asks whether a
new version beats the old one. Direct votes avoid the bias: in simulations
from our fit, about 260 give their win rate to within 3 points
(root-mean-square), while votes spread over other opponents stay 2.5 points
off at any budget (\hyperref[app:release]{Appendix~\ref*{app:release}}). Correcting scores gains
little per vote. Reward models use the same rule,
but score an answer given its prompt, $r(x,y)$, not a model across prompts, so an
aggregation-driven failure of model-level scores need not carry over. The Arena result motivates the
hypothesis that it partly does, since two samples of one policy are near-copies on
the same prompt; one $\beta$ in direct preference optimization would then misstate
their reward gaps. This remains unresolved: in human summarization
preferences \citep{stiennon2020learning}, same-policy pairs give $k=1.17$ (1.08 to
1.24), but not in the training split alone (\hyperref[app:rm]{Appendix~\ref*{app:rm}}).

\paragraph{Limitations.} That one score per
model misprices relatives in human Arena votes is shown directly; why is not:
correlated noise (\cref{prop:thurstone}) and comparison on differences fit the
pattern. The judges' menu effect is weaker within prompts, and in human choices clear
for one version pair only; Arena prompts are judged once, so only topics can be
controlled there. The Arena effect rests on 17
version and 49 size pairs from 12 and 21 families. Randomized human pairings in MultiPref
\citep{miranda2024hybrid} point the same way (size 1.47, same lab 1.11) but are
too few to confirm it (\hyperref[app:multipref]{Appendix~\ref*{app:multipref}}). Arena votes are observational; timing checks
and placebos address obvious selection. We cannot separate per-vote similarity from voter
heterogeneity over attributes that siblings share, nor, beyond coarse prompt
strata, answer similarity from the prompt mix of each pair. The lineage labels
were fixed in advance, but the similarity measures and model forms were developed
on these data; the frozen transfer test is retrospective, not confirmatory.

\section*{Software and Data}
All data are public (Arena battle log, the 100k and 140k human-preference
releases, PRISM, MultiPref and the summarize-from-feedback comparisons). Code that downloads the data and rebuilds every number, table and
figure, together with the generated answers and judgments of
\cref{sec:judges}, will be released.

\section*{Impact Statement}
This paper studies how public leaderboards and preference-trained reward models
turn pairwise judgments into scores. Its main practical consequence is that
reported head-to-head comparisons between related model versions can be
miscalibrated; correcting them makes evaluation claims more accurate. We use only
public, anonymized voting and rating data and synthetic judgments from language models; no
new human data were collected.

\bibliography{references}

\begin{thebibliography}{52}
\providecommand{\natexlab}[1]{#1}
\providecommand{\url}[1]{\texttt{#1}}
\expandafter\ifx\csname urlstyle\endcsname\relax
  \providecommand{\doi}[1]{doi: #1}\else
  \providecommand{\doi}{doi: \begingroup \urlstyle{rm}\Url}\fi

\bibitem[Ameli et~al.(2025)Ameli, Zhuang, Stoica, and
  Mahoney]{ameli2025statistical}
Ameli, S., Zhuang, S., Stoica, I., and Mahoney, M.~W.
\newblock A statistical framework for ranking {LLM}-based chatbots.
\newblock In \emph{International Conference on Learning Representations
  (ICLR)}, 2025.

\bibitem[Becker et~al.(1963)Becker, DeGroot, and
  Marschak]{becker1963probabilities}
Becker, G.~M., DeGroot, M.~H., and Marschak, J.
\newblock Probabilities of choices among very similar objects: An experiment to
  decide between two models.
\newblock \emph{Behavioral Science}, 8\penalty0 (4):\penalty0 306--311, 1963.

\bibitem[Benson et~al.(2016)Benson, Kumar, and Tomkins]{benson2016relevance}
Benson, A.~R., Kumar, R., and Tomkins, A.
\newblock On the relevance of irrelevant alternatives.
\newblock In \emph{Proceedings of the 25th International Conference on World
  Wide Web (WWW)}, pp.\  963--973, 2016.

\bibitem[Boubdir et~al.(2024)Boubdir, Kim, Ermis, Hooker, and
  Fadaee]{boubdir2024elo}
Boubdir, M., Kim, E., Ermis, B., Hooker, S., and Fadaee, M.
\newblock Elo uncovered: Robustness and best practices in language model
  evaluation.
\newblock In \emph{Advances in Neural Information Processing Systems},
  volume~37, pp.\  106135--106161, 2024.
\newblock Earlier version in Proceedings of the Third Workshop on Natural
  Language Generation, Evaluation, and Metrics (GEM), 2023, pp. 339--352.

\bibitem[Bower \& Balzano(2020)Bower and Balzano]{bower2020preference}
Bower, A. and Balzano, L.
\newblock Preference modeling with context-dependent salient features.
\newblock In \emph{Proceedings of the 37th International Conference on Machine
  Learning (ICML)}, volume 119 of \emph{Proceedings of Machine Learning
  Research}, pp.\  1067--1077, 2020.

\bibitem[Bradley \& Terry(1952)Bradley and Terry]{bradley1952rank}
Bradley, R.~A. and Terry, M.~E.
\newblock Rank analysis of incomplete block designs: {I}. {T}he method of
  paired comparisons.
\newblock \emph{Biometrika}, 39\penalty0 (3/4):\penalty0 324--345, 1952.

\bibitem[Busemeyer \& Townsend(1993)Busemeyer and
  Townsend]{busemeyer1993decision}
Busemeyer, J.~R. and Townsend, J.~T.
\newblock Decision field theory: A dynamic-cognitive approach to decision
  making in an uncertain environment.
\newblock \emph{Psychological Review}, 100\penalty0 (3):\penalty0 432--459,
  1993.

\bibitem[Chen \& Joachims(2016)Chen and Joachims]{chen2016modeling}
Chen, S. and Joachims, T.
\newblock Modeling intransitivity in matchup and comparison data.
\newblock In \emph{Proceedings of the Ninth ACM International Conference on Web
  Search and Data Mining (WSDM)}, pp.\  227--236, 2016.

\bibitem[Cherapanamjeri et~al.(2026)Cherapanamjeri, Daskalakis, Farina, and
  Mohammadpour]{cherapanamjeri2026learning}
Cherapanamjeri, Y., Daskalakis, C., Farina, G., and Mohammadpour, S.
\newblock Learning correlated reward models: Statistical barriers and
  opportunities.
\newblock In \emph{International Conference on Learning Representations
  (ICLR)}, 2026.

\bibitem[Chiang et~al.(2024)Chiang, Zheng, Sheng, Angelopoulos, Li, Li, Zhu,
  Zhang, Jordan, Gonzalez, and Stoica]{chiang2024chatbot}
Chiang, W.-L., Zheng, L., Sheng, Y., Angelopoulos, A.~N., Li, T., Li, D., Zhu,
  B., Zhang, H., Jordan, M.~I., Gonzalez, J.~E., and Stoica, I.
\newblock Chatbot arena: An open platform for evaluating {LLM}s by human
  preference.
\newblock In \emph{Proceedings of the 41st International Conference on Machine
  Learning (ICML)}, volume 235 of \emph{Proceedings of Machine Learning
  Research}, pp.\  8359--8388, 2024.

\bibitem[Christiano et~al.(2017)Christiano, Leike, Brown, Martic, Legg, and
  Amodei]{christiano2017deep}
Christiano, P.~F., Leike, J., Brown, T., Martic, M., Legg, S., and Amodei, D.
\newblock Deep reinforcement learning from human preferences.
\newblock In \emph{Advances in Neural Information Processing Systems}, 2017.

\bibitem[Davidson(1970)]{davidson1970extending}
Davidson, R.~R.
\newblock On extending the {B}radley-{T}erry model to accommodate ties in
  paired comparison experiments.
\newblock \emph{Journal of the American Statistical Association}, 65\penalty0
  (329):\penalty0 317--328, 1970.

\bibitem[Debreu(1960)]{debreu1960review}
Debreu, G.
\newblock Review of {R}. {D}. {L}uce, {I}ndividual choice behavior: A
  theoretical analysis.
\newblock \emph{American Economic Review}, 50:\penalty0 186--188, 1960.

\bibitem[Dubois et~al.(2024)Dubois, Liang, and Hashimoto]{dubois2024length}
Dubois, Y., Liang, P., and Hashimoto, T.~B.
\newblock Length-controlled {AlpacaEval}: A simple debiasing of automatic
  evaluators.
\newblock In \emph{Conference on Language Modeling (COLM)}, 2024.

\bibitem[Glenn \& David(1960)Glenn and David]{glenn1960ties}
Glenn, W.~A. and David, H.~A.
\newblock Ties in paired-comparison experiments using a modified
  {T}hurstone-{M}osteller model.
\newblock \emph{Biometrics}, 16\penalty0 (1):\penalty0 86--109, 1960.

\bibitem[Haghtalab et~al.(2026)Haghtalab, Procaccia, Shao, Wang, and
  Yang]{haghtalab2026pluralistic}
Haghtalab, N., Procaccia, A.~D., Shao, H., Wang, S.~L., and Yang, K.
\newblock Pluralistic leaderboards.
\newblock In \emph{Proceedings of the 43rd International Conference on Machine
  Learning (ICML)}, volume 306 of \emph{Proceedings of Machine Learning
  Research}, pp.\  39356--39375, 2026.

\bibitem[Huber et~al.(1982)Huber, Payne, and Puto]{huber1982adding}
Huber, J., Payne, J.~W., and Puto, C.
\newblock Adding asymmetrically dominated alternatives: Violations of
  regularity and the similarity hypothesis.
\newblock \emph{Journal of Consumer Research}, 9\penalty0 (1):\penalty0 90--98,
  1982.

\bibitem[Kirk et~al.(2024)Kirk, Whitefield, R{\"o}ttger, Bean, Margatina,
  Mosquera-Gomez, Ciro, Bartolo, Williams, He, Vidgen, and Hale]{kirk2024prism}
Kirk, H.~R., Whitefield, A., R{\"o}ttger, P., Bean, A.~M., Margatina, K.,
  Mosquera-Gomez, R., Ciro, J., Bartolo, M., Williams, A., He, H., Vidgen, B.,
  and Hale, S.~A.
\newblock The {PRISM} alignment dataset: What participatory, representative and
  individualised human feedback reveals about the subjective and multicultural
  alignment of large language models.
\newblock In \emph{Advances in Neural Information Processing Systems (Datasets
  and Benchmarks Track)}, 2024.

\bibitem[Li et~al.(2024)Li, Angelopoulos, and Chiang]{li2024stylecontrol}
Li, T., Angelopoulos, A., and Chiang, W.-L.
\newblock Does style matter? {D}isentangling style and substance in {Chatbot
  Arena}.
\newblock LMSYS Org blog,
  \url{https://lmsys.org/blog/2024-08-28-style-control/}, 2024.
\newblock Posted August 29, 2024.

\bibitem[{LMArena}(2024)]{lmarena2024hp100k}
{LMArena}.
\newblock Arena human preference 100k.
\newblock Hugging Face dataset
  \url{https://huggingface.co/datasets/lmarena-ai/arena-human-preference-100k},
  2024.

\bibitem[{LMArena}(2025)]{lmarena2025hp140k}
{LMArena}.
\newblock Arena human preference 140k.
\newblock Hugging Face dataset
  \url{https://huggingface.co/datasets/lmarena-ai/arena-human-preference-140k},
  CC-BY-4.0, 2025.

\bibitem[Luce(1959)]{luce1959individual}
Luce, R.~D.
\newblock \emph{Individual Choice Behavior: A Theoretical Analysis}.
\newblock Wiley, 1959.

\bibitem[McFadden(1974)]{mcfadden1974conditional}
McFadden, D.
\newblock Conditional logit analysis of qualitative choice behavior.
\newblock In Zarembka, P. (ed.), \emph{Frontiers in Econometrics}, pp.\
  105--142. Academic Press, New York, 1974.

\bibitem[McFadden(1978)]{mcfadden1978modeling}
McFadden, D.
\newblock Modeling the choice of residential location.
\newblock \emph{Transportation Research Record}, \penalty0 (673):\penalty0
  72--77, 1978.

\bibitem[McFadden \& Train(2000)McFadden and Train]{mcfadden2000mixed}
McFadden, D. and Train, K.
\newblock Mixed {MNL} models for discrete response.
\newblock \emph{Journal of Applied Econometrics}, 15\penalty0 (5):\penalty0
  447--470, 2000.

\bibitem[Mellers \& Biagini(1994)Mellers and Biagini]{mellers1994similarity}
Mellers, B.~A. and Biagini, K.
\newblock Similarity and choice.
\newblock \emph{Psychological Review}, 101\penalty0 (3):\penalty0 505--518,
  1994.

\bibitem[Miranda et~al.(2025)Miranda, Wang, Elazar, Kumar, Pyatkin, Brahman,
  Smith, Hajishirzi, and Dasigi]{miranda2024hybrid}
Miranda, L. J.~V., Wang, Y., Elazar, Y., Kumar, S., Pyatkin, V., Brahman, F.,
  Smith, N.~A., Hajishirzi, H., and Dasigi, P.
\newblock Hybrid preferences: Learning to route instances for human vs. {AI}
  feedback.
\newblock In \emph{Proceedings of the 63rd Annual Meeting of the Association
  for Computational Linguistics (Volume 1: Long Papers)}, pp.\  7162--7200,
  2025.

\bibitem[Munos et~al.(2024)Munos, Valko, Calandriello, Azar, Rowland, Guo,
  Tang, Geist, Mesnard, Fiegel, Michi, Selvi, Girgin, Momchev, Bachem,
  Mankowitz, Precup, and Piot]{munos2024nash}
Munos, R., Valko, M., Calandriello, D., Azar, M.~G., Rowland, M., Guo, Z.~D.,
  Tang, Y., Geist, M., Mesnard, T., Fiegel, C., Michi, A., Selvi, M., Girgin,
  S., Momchev, N., Bachem, O., Mankowitz, D.~J., Precup, D., and Piot, B.
\newblock {N}ash learning from human feedback.
\newblock In \emph{Proceedings of the 41st International Conference on Machine
  Learning (ICML)}, volume 235 of \emph{Proceedings of Machine Learning
  Research}, pp.\  36743--36768, 2024.

\bibitem[Ouyang et~al.(2022)Ouyang, Wu, Jiang, Almeida, Wainwright, Mishkin,
  Zhang, Agarwal, Slama, Ray, et~al.]{ouyang2022training}
Ouyang, L., Wu, J., Jiang, X., Almeida, D., Wainwright, C.~L., Mishkin, P.,
  Zhang, C., Agarwal, S., Slama, K., Ray, A., et~al.
\newblock Training language models to follow instructions with human feedback.
\newblock In \emph{Advances in Neural Information Processing Systems}, 2022.

\bibitem[Rafailov et~al.(2023)Rafailov, Sharma, Mitchell, Ermon, Manning, and
  Finn]{rafailov2023direct}
Rafailov, R., Sharma, A., Mitchell, E., Ermon, S., Manning, C.~D., and Finn, C.
\newblock Direct preference optimization: Your language model is secretly a
  reward model.
\newblock In \emph{Advances in Neural Information Processing Systems}, 2023.

\bibitem[Ragain \& Ugander(2016)Ragain and Ugander]{ragain2016pairwise}
Ragain, S. and Ugander, J.
\newblock Pairwise choice {M}arkov chains.
\newblock In \emph{Advances in Neural Information Processing Systems},
  volume~29, 2016.

\bibitem[Reimers \& Gurevych(2019)Reimers and Gurevych]{reimers2019sentence}
Reimers, N. and Gurevych, I.
\newblock Sentence-{BERT}: Sentence embeddings using {Siamese} {BERT}-networks.
\newblock In \emph{Proceedings of the 2019 Conference on Empirical Methods in
  Natural Language Processing and the 9th International Joint Conference on
  Natural Language Processing (EMNLP-IJCNLP)}, pp.\  3982--3992, 2019.

\bibitem[Roe et~al.(2001)Roe, Busemeyer, and Townsend]{roe2001multialternative}
Roe, R.~M., Busemeyer, J.~R., and Townsend, J.~T.
\newblock Multialternative decision field theory: A dynamic connectionist model
  of decision making.
\newblock \emph{Psychological Review}, 108\penalty0 (2):\penalty0 370--392,
  2001.

\bibitem[Rumelhart \& Greeno(1971)Rumelhart and
  Greeno]{rumelhart1971similarity}
Rumelhart, D.~L. and Greeno, J.~G.
\newblock Similarity between stimuli: An experimental test of the {L}uce and
  {R}estle choice models.
\newblock \emph{Journal of Mathematical Psychology}, 8\penalty0 (3):\penalty0
  370--381, 1971.

\bibitem[Seshadri \& Ugander(2019)Seshadri and
  Ugander]{seshadri2019fundamental}
Seshadri, A. and Ugander, J.
\newblock Fundamental limits of testing the independence of irrelevant
  alternatives in discrete choice.
\newblock In \emph{Proceedings of the 2019 ACM Conference on Economics and
  Computation (EC)}, pp.\  65--66, 2019.
\newblock Extended version: arXiv:2001.07042.

\bibitem[Seshadri et~al.(2019)Seshadri, Peysakhovich, and
  Ugander]{seshadri2019discovering}
Seshadri, A., Peysakhovich, A., and Ugander, J.
\newblock Discovering context effects from raw choice data.
\newblock In \emph{Proceedings of the 36th International Conference on Machine
  Learning (ICML)}, volume~97 of \emph{Proceedings of Machine Learning
  Research}, pp.\  5660--5669, 2019.

\bibitem[Singh et~al.(2025)Singh, Nan, Wang, et~al.]{singh2025leaderboard}
Singh, S., Nan, Y., Wang, A., et~al.
\newblock The leaderboard illusion.
\newblock In \emph{Advances in Neural Information Processing Systems (Datasets
  and Benchmarks Track)}, 2025.

\bibitem[Siththaranjan et~al.(2024)Siththaranjan, Laidlaw, and
  Hadfield-Menell]{siththaranjan2024distributional}
Siththaranjan, A., Laidlaw, C., and Hadfield-Menell, D.
\newblock Distributional preference learning: Understanding and accounting for
  hidden context in {RLHF}.
\newblock In \emph{International Conference on Learning Representations
  (ICLR)}, 2024.

\bibitem[Stiennon et~al.(2020)Stiennon, Ouyang, Wu, Ziegler, Lowe, Voss,
  Radford, Amodei, and Christiano]{stiennon2020learning}
Stiennon, N., Ouyang, L., Wu, J., Ziegler, D., Lowe, R., Voss, C., Radford, A.,
  Amodei, D., and Christiano, P.~F.
\newblock Learning to summarize with human feedback.
\newblock In \emph{Advances in Neural Information Processing Systems}, 2020.

\bibitem[Sun et~al.(2025)Sun, Shen, and Ton]{sun2025rethinking}
Sun, H., Shen, Y., and Ton, J.-F.
\newblock Rethinking reward modeling in preference-based large language model
  alignment.
\newblock In \emph{International Conference on Learning Representations
  (ICLR)}, 2025.
\newblock arXiv:2411.04991, titled ``Rethinking {Bradley-Terry} Models in
  Preference-Based Reward Modeling: Foundations, Theory, and Alternatives''.

\bibitem[Swamy et~al.(2024)Swamy, Dann, Kidambi, Wu, and
  Agarwal]{swamy2024minimaximalist}
Swamy, G., Dann, C., Kidambi, R., Wu, Z.~S., and Agarwal, A.
\newblock A minimaximalist approach to reinforcement learning from human
  feedback.
\newblock In \emph{Proceedings of the 41st International Conference on Machine
  Learning (ICML)}, volume 235 of \emph{Proceedings of Machine Learning
  Research}, pp.\  47345--47377, 2024.

\bibitem[Thurstone(1927)]{thurstone1927law}
Thurstone, L.~L.
\newblock A law of comparative judgment.
\newblock \emph{Psychological Review}, 34\penalty0 (4):\penalty0 273--286,
  1927.

\bibitem[Tomlinson \& Benson(2021)Tomlinson and Benson]{tomlinson2021learning}
Tomlinson, K. and Benson, A.~R.
\newblock Learning interpretable feature context effects in discrete choice.
\newblock In \emph{Proceedings of the 27th ACM SIGKDD Conference on Knowledge
  Discovery and Data Mining (KDD)}, pp.\  1582--1592, 2021.

\bibitem[Train(2009)]{train2009discrete}
Train, K.~E.
\newblock \emph{Discrete Choice Methods with Simulation}.
\newblock Cambridge University Press, 2nd edition, 2009.

\bibitem[Tversky(1972)]{tversky1972elimination}
Tversky, A.
\newblock Elimination by aspects: A theory of choice.
\newblock \emph{Psychological Review}, 79\penalty0 (4):\penalty0 281--299,
  1972.

\bibitem[Tversky \& Russo(1969)Tversky and Russo]{tversky1969substitutability}
Tversky, A. and Russo, J.~E.
\newblock Substitutability and similarity in binary choices.
\newblock \emph{Journal of Mathematical Psychology}, 6\penalty0 (1):\penalty0
  1--12, 1969.

\bibitem[Wang et~al.(2024)Wang, Li, Chen, Cai, Zhu, Lin, Cao, Kong, Liu, Liu,
  and Sui]{wang2023fair}
Wang, P., Li, L., Chen, L., Cai, Z., Zhu, D., Lin, B., Cao, Y., Kong, L., Liu,
  Q., Liu, T., and Sui, Z.
\newblock Large language models are not fair evaluators.
\newblock In \emph{Proceedings of the 62nd Annual Meeting of the Association
  for Computational Linguistics (Volume 1: Long Papers)}, pp.\  9440--9450,
  2024.

\bibitem[Wang et~al.(2020)Wang, Wei, Dong, Bao, Yang, and Zhou]{wang2020minilm}
Wang, W., Wei, F., Dong, L., Bao, H., Yang, N., and Zhou, M.
\newblock {MiniLM}: Deep self-attention distillation for task-agnostic
  compression of pre-trained transformers.
\newblock In \emph{Advances in Neural Information Processing Systems}, 2020.

\bibitem[Wu et~al.(2024)Wu, Xie, Yang, Wu, Gao, Ding, Wang, and
  He]{wu2024betadpo}
Wu, J., Xie, Y., Yang, Z., Wu, J., Gao, J., Ding, B., Wang, X., and He, X.
\newblock $\beta$-{DPO}: Direct preference optimization with dynamic $\beta$.
\newblock In \emph{Advances in Neural Information Processing Systems},
  volume~37, pp.\  129944--129966, 2024.

\bibitem[Xing(2026)]{xing2026attention}
Xing, W.
\newblock Attention limited reward learning.
\newblock \emph{arXiv preprint arXiv:2607.04590}, 2026.

\bibitem[Zhang et~al.(2025)Zhang, Zhang, Wu, Xu, and Gu]{zhang2024beyond}
Zhang, Y., Zhang, G., Wu, Y., Xu, K., and Gu, Q.
\newblock Beyond {Bradley-Terry} models: A general preference model for
  language model alignment.
\newblock In \emph{Proceedings of the 42nd International Conference on Machine
  Learning (ICML)}, volume 267 of \emph{Proceedings of Machine Learning
  Research}, pp.\  76939--76965, 2025.

\bibitem[Zheng et~al.(2023)Zheng, Chiang, Sheng, Zhuang, Wu, Zhuang, Lin, Li,
  Li, Xing, Zhang, Gonzalez, and Stoica]{zheng2023judging}
Zheng, L., Chiang, W.-L., Sheng, Y., Zhuang, S., Wu, Z., Zhuang, Y., Lin, Z.,
  Li, Z., Li, D., Xing, E.~P., Zhang, H., Gonzalez, J.~E., and Stoica, I.
\newblock Judging {LLM}-as-a-judge with {MT-Bench} and {Chatbot Arena}.
\newblock In \emph{Advances in Neural Information Processing Systems (Datasets
  and Benchmarks Track)}, 2023.

\end{thebibliography}
\bibliographystyle{icml2026}

\newpage
\appendix
\onecolumn

\section{Robustness}
\subsection{Lineage, Prompt and Timing Checks}
\label{app:robust}

\paragraph{Lineage judgment calls.} Dropping every related pair that involves one
of the thirteen models whose tag is a judgment call (2 version, 10 size and 24
same-lab pairs) gives \x1.44 (1.05 to 1.85), \x1.27 (1.17 to 1.34) and \x1.12
(1.08 to 1.18).

\paragraph{Prompts and thresholds.} English only, first turn only, de-duplicated
prompts, code and non-code, creative and technical prompts, 50 to 200 votes per
pair, and bins of score closeness all give version 1.31 to 1.44, size 1.12 to
1.29 and lab 1.08 to 1.13. The hardest and easiest prompts (six or more, and two
or fewer, of Arena's seven hard-prompt criteria) give \x1.29, \x1.09, \x1.11 and
\x1.20, \x1.51, \x1.20: all above 1, with the size effect the least stable.
Reweighting each pair's votes to the prompt mix of unrelated pairs (24 strata of
difficulty, code, language and turn) gives \x1.37 (1.16 to 1.79), \x1.25 and
\x1.13; relatives and unrelated pairs already see similar prompts. Fitting each calendar quarter separately gives
positive version effects in every quarter with version pairs and positive size
effects from late 2023 on (the two earlier quarters have two and eight size
pairs).

\paragraph{Timing and launch sampling.} Arena samples a new model heavily after
launch, often against its predecessor, and models and voters drift. If siblings
were mostly compared with their opponents in different weeks, pooled scores could
make them look decisive for reasons unrelated to similarity. Four checks argue
against it. With a separate score for every model in every week, so that a pair is only
compared with opponents voted on in the same week, the multipliers are \x1.36,
\x1.23 and \x1.13. Dropping each model's first four weeks makes the version
effect larger (\x1.64), not smaller. Unrelated pairs that entered the arena within
14 days of each other (387 pairs) get \x0.99, and the share of a pair's votes cast
at launch has no effect (\x0.99 per standard deviation). The same-week model-free
slopes are 1.28 (95\% interval 1.10 to 1.47), 1.24 and 1.11 for related pairs and
0.99 for unrelated pairs. The 140k release mostly agrees (week-local multipliers \x1.21, \x1.35 and
\x1.24), though dropping launch weeks there leaves too few size pairs to say much.

\begin{figure*}[h]
\centering
\includegraphics[width=\textwidth]{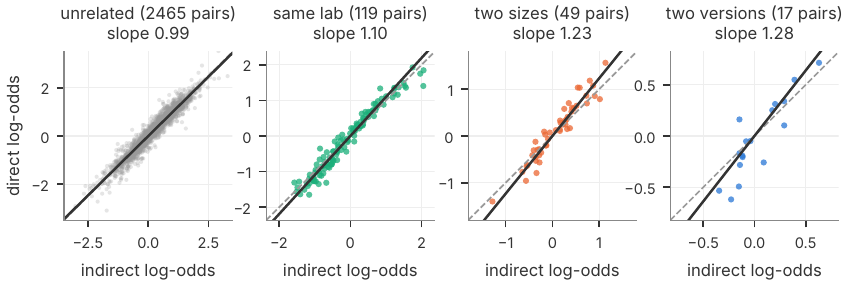}
\caption{Model-free check on the 2024 log. Each point is a model pair: the
direct head-to-head log-odds against the log-odds implied through shared
opponents. Dashed: Bradley--Terry (slope 1). Solid: inverse-variance weighted
fit through the origin. Unrelated pairs follow the product rule; related pairs
are more lopsided than the chain implies.}
\label{fig:direct_indirect}
\end{figure*}

\subsection{Prompt Topics}
\label{app:topics}

\paragraph{Clustering.} Each vote's first user turn (up to 2,000 characters) is
represented by TF-IDF (terms in at least five prompts, at most 50,000 terms,
sublinear term frequency, English stop words removed), reduced to 100 dimensions
by truncated SVD, normalized to unit length and clustered by $k$-means with 20 and
50 topics (three restarts, fixed seed). Topics hold 589 to 37,478 non-tied votes in
the 140k release and 357 to 20,986 in the 100k release. Only the 100k and 140k
releases carry prompts.

\paragraph{Topic profiles.} For each model and topic, the standardized residual
against plain Bradley--Terry fitted on all non-tied votes (sum of $y-\hat p$ over
the model's votes in the topic, divided by the binomial standard deviation),
leaving out the games between the two models being compared. For every pair of
models with at least 1,000 non-tied votes each, the correlation of the two
profiles across topics; related minus unrelated pairs, with 95\% intervals from
resampling topics (200 draws).

\paragraph{Topic-specific scores.} Bradley--Terry with one multiplier per relation
on non-tied votes of pairs with at least 30 (140k) or 50 (100k) votes, fitted with
$\operatorname{logit}P(a\succ b)=k_{\mathrm{rel}}\,(s_a+u_{a,c}-s_b-u_{b,c})$ for
prompt topic $c$ and a ridge $\rho\sum u^2$ ($\rho=10$, and 1 as a check), against
the same model with $u=0$. Intervals resample model pairs (200 draws).

\paragraph{Results.} With 20 topics and $\rho=10$, the multipliers with
topic-specific scores are, for version, size and same lab, \x1.46 (1.08 to an
unstable upper end: only 9 version pairs), \x1.20 (1.02 to 1.52) and \x1.24 (1.16
to 1.53) in the 140k release, against \x1.24, \x1.15 and \x1.17 with one score;
the ratio of the two is 1.00 to 1.21 for size and 1.05 to 1.17 for same lab. In
the 100k release they are \x1.65, \x1.40 (1.08 to 2.03) and \x1.23 (1.13 to 1.54),
against \x1.57, \x1.30 and \x1.19. With 50 topics they are larger still. Topics
are coarse: prompt heterogeneity within a topic remains, so the result shows that
the multipliers persist after controlling for these topics, not that prompt
pooling plays no part (\texttt{results/*/prompt\_topics.txt}).

\subsection{Sibling Pairs in the 2024 Test}
\label{app:pairs}
\Cref{tab:pairs} lists the model families behind the version and size pairs of the
2024 test, with their lineage tags from \texttt{src/lineage.py}. Every pair,
including the 119 same-lab pairs, is in \texttt{results/2024/pairs.csv}.

\begin{table}[tbp]
\caption{Version and size pairs in the 2024 test: model pairs with at least 50
non-tied votes in the battle log, grouped by lab and generation. A version pair
joins two releases of one model line, a size pair two lines of one generation.
Pairs: qualifying pairs, out of all pairs of that relation among the listed
models when some fall short of 50 votes.}
\label{tab:pairs}
\centering\footnotesize
\begin{tabular}{@{}ll>{\raggedright\arraybackslash}p{0.66\linewidth}r@{}}
\toprule
lab & generation & models (Arena names) & pairs \\
\midrule
\multicolumn{4}{@{}l}{\emph{version pairs (17, 12 families)}} \\
01.AI & \texttt{yi-large} & \texttt{yi-large}, \texttt{yi-large-preview} & 1 \\
Anthropic & \texttt{claude-2} & \texttt{claude-2.0}, \texttt{claude-2.1} & 1 \\
Google & \texttt{gemini-1.0} & \texttt{bard-jan-24-gemini-pro}, \texttt{gemini-pro-dev-api} & 1 \\
Google & \texttt{gemini-1.5} & \texttt{gemini-1.5-pro-api-0514}, \texttt{gemini-1.5-pro-exp-0801} & 1 \\
Hugging Face & \texttt{zephyr} & \texttt{zephyr-7b-alpha}, \texttt{zephyr-7b-beta} & 1 \\
Microsoft & \texttt{phi-3} & \texttt{phi-3-mini-128k-instruct}, \texttt{phi-3-mini-4k-instruct} & 1 \\
Nexusflow & \texttt{starling} & \texttt{starling-lm-7b-alpha}, \texttt{starling-lm-7b-beta} & 1 \\
OpenAI & \texttt{gpt-3.5} & \texttt{gpt-3.5-turbo-0125}, \texttt{gpt-3.5-turbo-0613}, \texttt{gpt-3.5-turbo-1106} & 3 \\
OpenAI & \texttt{gpt-4} & \texttt{gpt-4-0314}, \texttt{gpt-4-0613} & 1 \\
OpenAI & \texttt{gpt-4-turbo} & \texttt{gpt-4-0125-preview}, \texttt{gpt-4-1106-preview}, \texttt{gpt-4-turbo-2024-04-09} & 3 \\
OpenAI & \texttt{gpt-4o} & \texttt{chatgpt-4o-latest}, \texttt{gpt-4o-2024-05-13}, \texttt{gpt-4o-2024-08-06} & 2 of 3 \\
OpenChat & \texttt{openchat} & \texttt{openchat-3.5}, \texttt{openchat-3.5-0106} & 1 \\
\midrule
\multicolumn{4}{@{}l}{\emph{size pairs (49, 21 families)}} \\
Alibaba & \texttt{qwen1.5} & \texttt{qwen1.5-110b-chat}, \texttt{qwen1.5-14b-chat}, \texttt{qwen1.5-32b-chat}, \texttt{qwen1.5-4b-chat}, \texttt{qwen1.5-72b-chat}, \texttt{qwen1.5-7b-chat} & 9 of 15 \\
Anthropic & \texttt{claude-1} & \texttt{claude-1}, \texttt{claude-instant-1} & 1 \\
Anthropic & \texttt{claude-3} & \texttt{claude-3-haiku-20240307}, \texttt{claude-3-opus-20240229}, \texttt{claude-3-sonnet-20240229} & 3 \\
Cohere & \texttt{command-r} & \texttt{command-r}, \texttt{command-r-plus} & 1 \\
DeepSeek & \texttt{deepseek-v2} & \texttt{deepseek-coder-v2}, \texttt{deepseek-v2-api-0628} & 1 \\
Google & \texttt{gemini-1.5} & \texttt{gemini-1.5-flash-api-0514}, \texttt{gemini-1.5-pro-api-0514}, \texttt{gemini-1.5-pro-exp-0801}, \texttt{gemini-advanced-0514} & 5 \\
Google & \texttt{gemma-1} & \texttt{gemma-1.1-2b-it}, \texttt{gemma-1.1-7b-it}, \texttt{gemma-2b-it}, \texttt{gemma-7b-it} & 2 of 4 \\
Google & \texttt{gemma-2} & \texttt{gemma-2-27b-it}, \texttt{gemma-2-2b-it}, \texttt{gemma-2-9b-it} & 3 \\
LMSYS & \texttt{vicuna} & \texttt{vicuna-13b}, \texttt{vicuna-33b}, \texttt{vicuna-7b} & 3 \\
Meta & \texttt{llama-2} & \texttt{llama-2-13b-chat}, \texttt{llama-2-70b-chat}, \texttt{llama-2-7b-chat} & 3 \\
Meta & \texttt{llama-3} & \texttt{llama-3-70b-instruct}, \texttt{llama-3-8b-instruct} & 1 \\
Meta & \texttt{llama-3.1} & \texttt{llama-3.1-405b-instruct}, \texttt{llama-3.1-70b-instruct}, \texttt{llama-3.1-8b-instruct} & 3 \\
Microsoft & \texttt{phi-3} & \texttt{phi-3-medium-4k-instruct}, \texttt{phi-3-mini-4k-instruct}, \texttt{phi-3-mini-4k-instruct-june-2024}, \texttt{phi-3-small-8k-instruct} & 4 of 5 \\
Microsoft & \texttt{wizardlm} & \texttt{wizardlm-13b}, \texttt{wizardlm-70b} & 1 \\
Mistral & \texttt{mistral-api} & \texttt{mistral-medium}, \texttt{mistral-next} & 1 \\
Mistral & \texttt{mixtral} & \texttt{mixtral-8x22b-instruct-v0.1}, \texttt{mixtral-8x7b-instruct-v0.1} & 1 \\
MosaicML & \texttt{mpt} & \texttt{mpt-30b-chat}, \texttt{mpt-7b-chat} & 1 \\
OpenAI & \texttt{gpt-4o} & \texttt{chatgpt-4o-latest}, \texttt{gpt-4o-2024-05-13}, \texttt{gpt-4o-2024-08-06}, \texttt{gpt-4o-mini-2024-07-18} & 3 \\
Perplexity & \texttt{pplx} & \texttt{pplx-70b-online}, \texttt{pplx-7b-online} & 1 \\
Reka & \texttt{reka-flash} & \texttt{reka-flash-21b-20240226}, \texttt{reka-flash-21b-20240226-online} & 1 \\
Zhipu & \texttt{chatglm} & \texttt{chatglm2-6b}, \texttt{chatglm3-6b} & 1 \\
\bottomrule
\end{tabular}

\end{table}

\section{Held-Out and Transfer Tests}
\subsection{Further Held-Out Results}
\label{app:heldout}

Gains over plain Bradley--Terry in millinats per vote, over all test votes,
votes between relatives, and version and size pairs only (means; rivals tuned on validation data; seeds or cutoffs in
\texttt{results/*/fix\_heldout\_validated.csv}, 95\% pair-resampling intervals
for every model and every difference between models in
\texttt{fix\_heldout\_validated\_ci.csv}).

\begin{center}\small
\begin{tabular}{@{}llccc@{}}
\toprule
Release & Model & random votes & unseen pairs & later votes \\
\midrule
2024 & Lineage fix & +0.96 / +6.6 / +11.0 & +1.13 / +8.4 / +13.9 & +0.58 / +11.2 / +17.8 \\
2024 & Covariance + lineage & +1.00 / +6.9 / +10.9 & +0.52 / +7.6 / +12.3 & +0.66 / +12.2 / +17.7 \\
140k & Similarity fix (wording, length gap) & +1.63 / +8.0 / +14.0 & +0.65 / +3.4 / +9.8 & +1.24 / +4.9 / +7.8 \\
140k & Similarity fix with style & +1.70 / +8.3 / +14.2 & +0.78 / +3.6 / +9.6 & +1.25 / +5.0 / +7.1 \\
140k & Covariance + lineage & +0.08 / +5.1 / +9.9 & $-$1.71 / +1.3 / +5.8 & $-$1.02 / +4.3 / +6.6 \\
\bottomrule
\end{tabular}
\end{center}

On the 2024 log the lineage fix's gain on relatives ranges from +6.0 to +7.5
across vote seeds, +4.5 to +12.4 across pair seeds and +8.7 to +14.5 across time
cutoffs.

\subsection{Ties versus Win Odds in the Held-Out Tests}
\label{app:ablation}

\begin{table}[h]
\caption{Held-out gain in three-outcome log loss over plain Bradley--Terry over
all test votes and over votes between relatives (millinats per vote; means over
seeds or cutoffs; rivals tuned on validation data). Sources: \texttt{results/*/fix\_heldout\_validated.csv}.}
\label{tab:heldout}
\centering\small
\setlength{\tabcolsep}{4pt}
\begin{tabular}{@{}lccc@{}}
\toprule
 & random votes & unseen pairs & later votes \\
\midrule
\multicolumn{4}{@{}l}{\emph{2024 log (129 models)}} \\
Lineage fix & \textbf{+0.96 / +6.6} & \textbf{+1.13 / +8.4} & \textbf{+0.58 / +11.2} \\
Covariance & +0.04 / +0.4 & $-$0.78 / $-$0.1 & +0.03 / +0.5 \\
Gen.\ preference & +0.18 / +0.4 & $-$0.05 / +0.2 & +0.06 / +0.6 \\
Per-pair scale & +0.09 / +0.4 & -- & -- \\
\midrule
\multicolumn{4}{@{}l}{\emph{140k release (53 models)}} \\
Lineage fix & +0.83 / +5.7 & +0.49 / +3.0 & +0.51 / +4.8 \\
Similarity fix & \textbf{+1.70 / +8.3} & \textbf{+0.78 / +3.6} & \textbf{+1.25 / +5.0} \\
Covariance & $-$0.73 / $-$0.5 & $-$2.49 / $-$3.0 & $-$1.76 / $-$0.8 \\
Gen.\ preference & $-$0.41 / $-$0.0 & $-$0.81 / $-$0.5 & $-$0.14 / +0.2 \\
\bottomrule
\end{tabular}
\end{table}

\Cref{tab:heldout} gives the gains over all votes. The fixes' overall gain is
positive at every seed and cutoff, and the lineage fix gains on votes between
relatives at every seed and cutoff in both releases (the similarity fix at all
but one unseen-pair seed, where it gives $-0.04$); on unrelated pairs the 2024
time split loses slightly ($-0.17$ to $-0.89$). The gain between relatives is 4
to 20 times the overall gain. Tuning the rivals on the test set instead changes
their gains by at most 0.4. The covariance model with the six lineage parameters
gives $+1.00$ / $+6.9$ on random 2024 votes, slightly above the lineage fix on
later votes, and falls short of both fixes in the 140k release.

\Cref{tab:ablation} switches each part of the fixes on alone, with the splits,
seeds and cutoffs of \cref{sec:fix} and no tuning
(\texttt{results/*/fix\_ablation.txt}). ``Win odds'' is the log loss of the
winner on non-tied test votes, $-\log\sigma(\hat z)$, to which the tie terms do
not contribute except through the fitted scores. Intervals resample model pairs.
A decisiveness term alone barely changes three-outcome loss because a steeper
gap also lowers the predicted tie probability, which relatives raise. The
calibration slope is the coefficient $c$ in $\operatorname{logit}P = a + c\,\hat z$
fitted to non-tied test votes between relatives ($c>1$: predictions too close to
even).

\begin{table*}[h]
\caption{Each part of the fixes alone: held-out gain over plain Bradley--Terry on
votes between relatives, three-outcome / win-odds log loss (millinats per vote,
pooled over seeds or cutoffs), with 95\% intervals from resampling model pairs
below each row, and the calibration slope on non-tied votes between relatives
(200 pair draws). The similarity fix is the six-parameter version of
\cref{sec:fix} (wording, length gap and style).}
\label{tab:ablation}
\centering\small
\setlength{\tabcolsep}{4pt}
\begin{tabular}{@{}lccc@{}}
\toprule
 & random votes & unseen pairs & later votes \\
\midrule
\multicolumn{4}{@{}l}{\emph{2024 log}} \\
lineage, ties only & +6.3 / 0.00 & +8.0 / $-$0.02 & +10.9 / +0.04 \\
{\scriptsize 95\% interval} & {\scriptsize (+4.2, +8.4) / ($-$0.03, +0.02)} & {\scriptsize (+4.7, +11.6) / ($-$0.06, +0.01)} & {\scriptsize (+7.6, +13.8) / ($-$0.01, +0.08)} \\
lineage, decisiveness only & +0.05 / +0.29 & $-$0.01 / +0.22 & $-$0.24 / +0.14 \\
{\scriptsize 95\% interval} & {\scriptsize ($-$0.04, +0.14) / (+0.13, +0.43)} & {\scriptsize ($-$0.14, +0.11) / (+0.03, +0.41)} & {\scriptsize ($-$0.54, $-$0.03) / ($-$0.33, +0.51)} \\
lineage, both & +6.6 / +0.33 & +8.2 / +0.17 & +10.9 / $-$0.05 \\
{\scriptsize 95\% interval} & {\scriptsize (+4.5, +8.7) / ($-$0.03, +0.65)} & {\scriptsize (+4.8, +11.8) / ($-$0.28, +0.58)} & {\scriptsize (+7.6, +14.1) / ($-$1.01, +0.71)} \\
slope, plain BT & 1.09 (1.06, 1.14) & 1.09 (1.03, 1.15) & 1.05 (0.98, 1.12) \\
slope, lineage, decisiveness only & 1.04 (1.00, 1.08) & 1.03 (0.97, 1.10) & 0.99 (0.91, 1.07) \\
slope, lineage, both & 0.98 (0.94, 1.01) & 0.96 (0.91, 1.01) & 0.93 (0.85, 0.99) \\
\midrule
\multicolumn{4}{@{}l}{\emph{140k release}} \\
lineage, ties only & +5.4 / $-$0.01 & +3.2 / $-$0.04 & +5.0 / $-$0.01 \\
{\scriptsize 95\% interval} & {\scriptsize (+2.6, +8.3) / ($-$0.04, +0.01)} & {\scriptsize ($-$0.04, +6.9) / ($-$0.07, $-$0.01)} & {\scriptsize (+1.4, +9.0) / ($-$0.04, +0.01)} \\
lineage, decisiveness only & $-$0.07 / +0.09 & $-$0.42 / $-$0.34 & $-$0.32 / $-$0.12 \\
{\scriptsize 95\% interval} & {\scriptsize ($-$0.37, +0.25) / ($-$0.35, +0.57)} & {\scriptsize ($-$1.0, +0.13) / ($-$1.21, +0.45)} & {\scriptsize ($-$0.87, +0.23) / ($-$0.87, +0.60)} \\
lineage, both & +5.7 / +0.17 & +2.8 / $-$0.80 & +4.8 / $-$0.36 \\
{\scriptsize 95\% interval} & {\scriptsize (+2.7, +8.7) / ($-$0.68, +1.09)} & {\scriptsize ($-$0.58, +6.5) / ($-$2.37, +0.59)} & {\scriptsize (+0.86, +8.7) / ($-$1.75, +0.99)} \\
slope, plain BT & 1.09 (0.99, 1.22) & 1.02 (0.88, 1.16) & 1.05 (0.93, 1.18) \\
slope, lineage, decisiveness only & 1.03 (0.93, 1.14) & 0.94 (0.79, 1.05) & 0.97 (0.87, 1.11) \\
slope, lineage, both & 0.96 (0.87, 1.07) & 0.88 (0.76, 0.98) & 0.90 (0.81, 1.03) \\
similarity, ties only & +7.6 / $-$0.07 & +2.5 / $-$0.07 & +5.5 / $-$0.07 \\
{\scriptsize 95\% interval} & {\scriptsize (+4.5, +10.8) / ($-$0.12, $-$0.02)} & {\scriptsize ($-$0.57, +5.9) / ($-$0.17, +0.03)} & {\scriptsize (+2.0, +8.6) / ($-$0.16, +0.01)} \\
similarity, decisiveness only & +0.21 / +0.86 & +0.33 / +0.72 & $-$0.65 / $-$0.15 \\
{\scriptsize 95\% interval} & {\scriptsize ($-$0.30, +0.76) / (+0.11, +1.69)} & {\scriptsize ($-$0.60, +1.2) / ($-$0.63, +1.94)} & {\scriptsize ($-$2.0, +0.50) / ($-$1.89, +1.36)} \\
similarity, both & +8.3 / +1.05 & +3.3 / +0.72 & +5.1 / $-$0.38 \\
{\scriptsize 95\% interval} & {\scriptsize (+4.9, +11.6) / ($-$0.03, +2.22)} & {\scriptsize (+0.01, +6.7) / ($-$1.02, +2.38)} & {\scriptsize (+1.2, +8.7) / ($-$2.81, +1.70)} \\
slope, similarity, decisiveness only & 1.02 (0.93, 1.14) & 0.93 (0.80, 1.04) & 0.95 (0.82, 1.08) \\
slope, similarity, both & 0.99 (0.89, 1.09) & 0.90 (0.79, 1.03) & 0.93 (0.80, 1.04) \\
\bottomrule
\end{tabular}

\end{table*}

On all non-tied votes, the decisiveness term alone gains $+0.05$ ($0.03$ to
$0.07$) on random and $+0.04$ ($0.01$ to $0.06$) on unseen 2024 pairs, and the
similarity term $+0.39$ ($0.20$ to $0.56$) on random 2025 votes; the other cells'
intervals include zero. In the 140k release the slopes' intervals include 1 for
plain scores as well.

\subsection{Frozen Transfer Test}
\label{app:transfer}

The specification, the code and the primary outcome were committed before the
first run (\texttt{src/transfer.py}; \texttt{results/transfer.txt}). Each setting
learns a correction on source votes (all models and votes), freezes its pair
terms, fits one score and one tie tendency per model and a common tie level on
target votes between unrelated models only, and predicts every target vote between
related models. Corrections: the lineage fix; a single temperature, one multiplier
shared by all related pairs with the same tie offsets; tie offsets alone; and the
six-parameter similarity fix, with features standardized by the source mean and
standard deviation and a target pair's features taken from the answer text of its
battles. Settings: the 2024 log to the 140k release; the 100k release to the 140k
release (similarity); and the 2024 log cut at the 60th, 70th and 80th percentile
of time, predicting related pairs with at least one model first seen after the
cut. All cuts are reported. Gains are over plain Bradley--Terry fitted the same
way; intervals resample labs (every related pair belongs to one lab). The rmse is
the noise-corrected error of the predicted head-to-head win rate over pairs with
at least 10 non-tied votes, each pair weighted equally; ``votes worth'' is the
number of head-to-head votes whose sampling error equals it. Tie offsets alone give
the same win odds as plain scores, so their win-odds gain is zero by
construction; their three-outcome gains are $+9.4$, $+7.9$, $+8.9$ and $+8.0$.

\begin{table*}[h]
\caption{Frozen transfer: gains over plain Bradley--Terry on target votes between
related models (millinats per vote; 95\% intervals resampling labs), calibration
slope on non-tied votes, rmse of the predicted head-to-head win rate (percentage
points) and the number of head-to-head votes it is worth, plain $\to$ corrected.}
\label{tab:transfer}
\centering\scriptsize
\setlength{\tabcolsep}{3pt}
\begin{tabular}{@{}lllcccccc@{}}
\toprule
setting & pairs / labs & correction & three-way gain & win-odds gain & slope & rmse (pp) & votes worth \\
\midrule
2024 log $\to$ 140k & 125 / 9 & lineage & +9.74 ($-$0.48, +19.30) & +0.14 ($-$2.92, +1.68) & 1.11 $\to$ 0.95 & 3.1 $\to$ 3.2 & 240 $\to$ 219 \\
 & 125 / 9 & temperature & +9.84 ($-$0.20, +19.54) & +0.36 ($-$2.19, +1.59) & 1.11 $\to$ 0.96 & 3.1 $\to$ 3.0 & 240 $\to$ 248 \\
100k $\to$ 140k & 125 / 9 & similarity & +14.00 (+1.16, +27.29) & +1.39 ($-$1.32, +3.02) & 1.11 $\to$ 1.00 & 3.1 $\to$ 2.8 & 240 $\to$ 290 \\
2024 log, new releases (60th pct.) & 86 / 10 & lineage & +7.76 (+5.86, +10.51) & $-$0.64 ($-$1.51, +1.73) & 1.03 $\to$ 0.88 & 2.2 $\to$ 2.8 & 432 $\to$ 277 \\
 & 86 / 10 & temperature & +7.65 (+5.76, +10.38) & $-$0.86 ($-$1.77, +1.63) & 1.03 $\to$ 0.87 & 2.2 $\to$ 3.0 & 432 $\to$ 244 \\
2024 log, new releases (70th pct.) & 68 / 10 & lineage & +9.26 (+7.73, +11.48) & +0.38 ($-$0.36, +1.69) & 1.08 $\to$ 0.92 & 2.0 $\to$ 2.3 & 546 $\to$ 409 \\
 & 68 / 10 & temperature & +9.12 (+7.63, +11.32) & +0.15 ($-$0.72, +1.57) & 1.08 $\to$ 0.92 & 2.0 $\to$ 2.5 & 546 $\to$ 349 \\
2024 log, new releases (80th pct.) & 57 / 7 & lineage & +8.33 (+6.79, +9.48) & +0.38 ($-$0.17, +1.72) & 1.07 $\to$ 0.92 & 2.5 $\to$ 2.8 & 362 $\to$ 277 \\
 & 57 / 7 & temperature & +8.24 (+6.66, +9.42) & +0.24 ($-$0.48, +1.88) & 1.07 $\to$ 0.92 & 2.5 $\to$ 2.9 & 362 $\to$ 263 \\
\bottomrule
\end{tabular}

\end{table*}

\subsection{Diagnosing the Frozen Multipliers' Overcorrection}
\label{app:why}

This check was designed after the transfer results were seen. Its three
predictions and the reliability, out-of-sample multiplier and thinning analyses
were committed before it was run; the per-relation and headroom parts were added
after the first results (\texttt{src/transfer\_why.py};
\texttt{results/transfer\_why.txt}). The lab-resampling intervals and the paired
bootstrap were added afterwards and committed, with their decision rule,
before they were run (\texttt{src/transfer\_why\_ci.py};
\texttt{results/transfer\_why\_ci.txt}). The check uses the two targets with the
most related votes: the 140k release (source: the 2024 log) and the 2024 log
after the 70th-percentile cut (pairs with a new release; source: votes before the
cut). Target scores $\hat s$ are plain Bradley--Terry fitted on the target's votes
between unrelated models, and $\hat g_{ij} = \hat s_i - \hat s_j$ for a related
pair. The out-of-sample multiplier $k^\ast$ is the $c$ that best fits the target's
related non-tied votes with $\operatorname{logit} P(i \succ j) = c\,\hat g_{ij}$.
Three explanations make different predictions.
\begin{itemize}
\item \emph{Noise (regression dilution).} With $\hat g = g + e$, the best
multiplier on $\hat g$ is about $k\lambda$, with reliability $\lambda = 1 - \sum n
\operatorname{var}(e) / \sum n \hat g^2$ ($n$: the pair's non-tied votes;
$\operatorname{var}(e)$ from 100 refits on unrelated votes resampled within each
pair). Then $\lambda$ is well below 1, and thinning the unrelated votes lowers
$\lambda$ and $k^\ast$ together.
\item \emph{A shift between data sets.} Then the target's own in-sample
multiplier, applied to $\hat g$, does not overcorrect, while the frozen one does.
\item \emph{In-sample estimation.} When scores and multipliers are fitted together,
sibling votes also shape the scores; then even the target's own in-sample
multiplier exceeds $k^\ast/\lambda$.
\end{itemize}
Intervals come from 1,000 draws that resample the target's related pairs either
by lab, as the transfer test does, or by pair; each draw refits $k^\ast$, the
target's in-sample multiplier and, on an independent draw of the source's related
pairs, the source multipliers. Unrelated votes are held fixed, since their scores
are reliable. A difference counts as established only if both its intervals
exclude zero.

\begin{table}[h]
\caption{Out-of-sample decisiveness on the two targets (multipliers on the
log-odds scale; 95\% intervals resampling pairs in parentheses, labs in
brackets). Gains: win-odds log loss over plain scores on related non-tied votes,
millinats per vote; ``best'' multipliers are fitted to the target's own related
votes, an in-sample upper bound for frozen shared or relation-specific
multipliers applied to fixed target scores, not for better preference prediction
in general. A few lab draws contain no version (5) or size (2) pair of the 2024
target and are left out of those intervals.}
\label{tab:why}
\centering\scriptsize
\begin{tabular}{lcc}
\toprule
 & 140k release & 2024, new releases \\
\midrule
related pairs (non-tied votes), labs & 125 (10,328), 9 & 68 (18,824), 10 \\
reliability $\lambda$ (all / 50\% / 25\% of votes) & 0.99 / 0.98 / 0.96--0.97 & 1.00 / 0.99 / 0.99 \\
$k^\ast$ (all votes) & 1.11 (1.02 to 1.22) [0.89 to 1.28] & 1.10 (1.03 to 1.17) [1.02 to 1.24] \\
$k^\ast$ with 50\% / 25\% of votes (two seeds) & 1.09--1.10 / 1.07--1.14 & 1.09--1.11 / 1.07--1.12 \\
$k^\ast$ version & 1.12 (0.41 to 1.89) [0.53 to 1.94] & 1.43 (0.91 to 2.02) [1.14 to 2.46] \\
$k^\ast$ size & 1.10 (0.95 to 1.30) [0.89 to 1.30] & 1.24 (1.12 to 1.39) [1.16 to 1.61] \\
$k^\ast$ same lab & 1.12 (1.00 to 1.25) [0.85 to 1.28] & 1.06 (0.98 to 1.13) [0.96 to 1.23] \\
target in-sample shared multiplier & 1.16 (1.07 to 1.26) [0.92 to 1.35] & 1.12 (1.05 to 1.19) [1.04 to 1.27] \\
source (frozen) shared multiplier & 1.15 (1.11 to 1.19) [1.07 to 1.20] & 1.17 (1.12 to 1.22) [1.07 to 1.22] \\
slope with target's own / frozen shared & 0.96 / 0.96 & 0.98 / 0.92 \\
frozen lineage multipliers (version / size / lab) & 1.41 / 1.25 / 1.12 & 1.41 / 1.29 / 1.14 \\
target in-sample lineage multipliers (same order) & 1.23 / 1.18 / 1.15 & 1.50 / 1.28 / 1.08 \\
gain, best shared / best per relation & $+0.40$ / $+0.40$ & $+0.34$ / $+0.61$ \\
gain, frozen shared / frozen lineage & $+0.36$ / $+0.14$ & $+0.15$ / $+0.38$ \\
mean $|\hat g|$ (logits) & 0.46 & 0.50 \\
\bottomrule
\end{tabular}
\end{table}

\begin{table}[h]
\caption{Paired bootstrap of the differences behind the explanation (95\%
intervals resampling pairs in parentheses, labs in brackets; 1,000 draws each).
None excludes zero under both schemes.}
\label{tab:whydiff}
\centering\scriptsize
\begin{tabular}{lcc}
\toprule
 & 140k release & 2024, new releases \\
\midrule
joint estimation: in-sample $-\,k^\ast/\lambda$ & $+0.034$ ($+0.011$ to $+0.074$) [$-0.016$ to $+0.075$] & $+0.019$ ($-0.001$ to $+0.039$) [$+0.007$ to $+0.043$] \\
shift: source shared $-$ target in-sample & $-0.01$ ($-0.12$ to $+0.09$) [$-0.22$ to $+0.23$] & $+0.05$ ($-0.03$ to $+0.13$) [$-0.12$ to $+0.14$] \\
overshoot: source shared $-\,k^\ast$ & $+0.04$ ($-0.08$ to $+0.15$) [$-0.16$ to $+0.25$] & $+0.07$ ($-0.01$ to $+0.15$) [$-0.09$ to $+0.17$] \\
source version $-\,k^\ast$ version & $+0.29$ ($-0.52$ to $+0.97$) [$-0.52$ to $+1.31$] & $-0.02$ ($-0.71$ to $+0.60$) [$-1.04$ to $+0.61$] \\
source size $-\,k^\ast$ size & $+0.15$ ($-0.07$ to $+0.32$) [$-0.07$ to $+0.37$] & $+0.05$ ($-0.12$ to $+0.17$) [$-0.33$ to $+0.16$] \\
source lab $-\,k^\ast$ lab & $+0.01$ ($-0.13$ to $+0.14$) [$-0.22$ to $+0.27$] & $+0.09$ ($-0.01$ to $+0.16$) [$-0.16$ to $+0.21$] \\
flattening, version $-$ lab (source $-$ target) & $+0.28$ ($-0.53$ to $+0.98$) [$-0.42$ to $+1.28$] & $-0.11$ ($-0.81$ to $+0.48$) [$-1.00$ to $+0.57$] \\
flattening, size $-$ lab (source $-$ target) & $+0.14$ ($-0.10$ to $+0.37$) [$-0.03$ to $+0.41$] & $-0.04$ ($-0.23$ to $+0.13$) [$-0.48$ to $+0.28$] \\
\bottomrule
\end{tabular}
\end{table}

\paragraph{Results.} Noise explains little: $\lambda$ is 0.99 and 1.00 (at least
0.96 with a quarter of the unrelated votes), and $k^\ast$ does not fall with
thinning beyond seed-to-seed variation (1.07 to 1.14), so dilution can account
for at most about one percentage point of the multipliers. On the observed
related pairs, sibling votes are steeper than unrelated-vote scores imply in both
targets ($k^\ast$ 1.11 and 1.10, pair intervals above 1), matching plain scores'
calibration slopes in \cref{tab:transfer} (1.11 and 1.08). Resampling labs, as the
transfer test does, the interval stays above 1 for the new 2024 releases (1.02 to
1.24) but not in the 140k release (0.89 to 1.28): at the level of labs,
persistence is shown in one target only.

The paired bootstrap (\cref{tab:whydiff}) establishes none of the differences
behind the explanation, so the attribution below rests on point estimates. Joint
estimation: the target's own in-sample multiplier exceeds $k^\ast/\lambda$ by
0.03 and 0.02, with intervals above zero under one scheme in each target; on the
140k release it overcorrects as much as the frozen one (slope 0.96), while the
source and target shared multipliers agree (1.15 and 1.16). A change between
periods: in the 2024 log the shared multiplier is 1.17 before the cut and 1.12
after (difference 0.05, intervals including zero). The frozen shared
multiplier's total overshoot (0.04 and 0.07) has intervals that include zero. In
the 140k release $k^\ast$ is flat across relations while the frozen lineage
multipliers fall from 1.41 to 1.12, which is why the lineage fix gains less there
than one shared multiplier ($+0.14$ against $+0.36$); with 13 version pairs, the
flattening's intervals are wide.

The attainable gain is small. Sibling gaps average about half a logit, a 10\%
steeper slope moves their predicted win rates by about one percentage point (1.0
and 1.1). Near even odds the expected log-loss gain from the right multiplier $k$
on a gap $g$ is a Bernoulli divergence, about $(k-1)^2 g^2/8$ per vote, so at these
gaps a 10\% correction is worth about 0.3 millinats, although it changes the
head-to-head odds a lab reads when comparing two checkpoints. Shared or relation-specific multipliers fitted to the target's own
sibling votes, applied to the fixed unrelated-vote scores, gain at most $0.40$
and $0.61$ millinats per vote. This bounds only that class of corrections, not
better preference prediction in general. Against lab intervals about 2 to 5
millinats wide in \cref{tab:transfer} (4.6 for the primary outcome), the transfer
evaluation lacks the precision to detect gains of this size; this compares
magnitudes and is not a formal power calculation.

\subsection{Votes for a Release Decision}
\label{app:release}
How should a lab spend votes to decide whether a new version beats its
predecessor? We simulate votes from the lineage model fitted to the 2024 log
(129 models; $k = 1.37$, $1.26$ and $1.12$ for versions, sizes and same-lab
pairs), treating $k$ and all other scores as known and ignoring ties. In each of
the 66 version and size pairs, the model that appeared later is the new one. It
receives $n$ decisive votes, and its win rate against the old one is estimated in
three ways, 400 replications each: (a) votes spread over its unrelated opponents
in Arena's own proportions, its score estimated by maximum likelihood and the win
rate read off plain Bradley--Terry; (b) the same votes read with the multiplier;
(c) all $n$ votes between the two, observed win rate. At $n = 100$, 300, 1,000 and
3,000 the root-mean-square errors are 5.8, 3.9, 3.0 and 2.7 percentage points for
(a), whose bias part stays at 2.5 to 2.6; 6.4, 3.7, 2.0 and 1.2 for (b); and 4.8,
2.8, 1.5 and 0.9 for (c). A direct vote carries Fisher information $k^2p(1-p)$
about the gap, against $p(1-p)$ for a vote against an equally strong unrelated
opponent. For the leaderboard as a whole the gain is small: over the pairs Arena
showed, the allocation minimizing the mean variance of all score gaps when $k$ is
known needs about 2.5\% fewer votes than the one that is optimal under plain
Bradley--Terry. These figures follow from the fitted model; they are not new
evidence.

\section{Judge Experiment Details}
\label{app:judges}

The judge saw a system message (``You are judging answers to a user's request.
Pick the single best answer: the most helpful, correct and well written. Reply
with one letter only.''), the user's request and the answers labeled A, B (and
C), and the question ``Which answer is best? Reply with exactly one letter: A or B.''
(``A, B or C'' for three answers). Answers were generated at temperature 0.7 with
a limit of 800 tokens; Qwen3 models
were asked not to think. Claude Haiku 5.5, Gemini 2.5 Flash-Lite and DeepSeek V3.2
were run with reasoning off. Of 138,400 judgments, 32 could not be read as a
letter and are dropped. The matched placebo $U$ is the model from a lab other than
those of $S_1$ and $T$ whose pairwise score, averaged over the first two judges,
is closest to that of $S_2$. \Cref{tab:judges} gives the per-judge results
(\texttt{results/judge/choice.txt}). The choice model has seven scores, a
first-slot term for pairs and first- and second-slot terms for triads, a lapse
rate, a scale for triads and five nest parameters (version, size and same-lab
pairs; version and size triads); intervals resample prompts (200 draws).
\texttt{results/judge/sensitivity.txt} repeats the matched test for the two
judges that did not choose the placebo ($+0.10$, $+0.02$ to $+0.17$; 24 of 32
designs above zero), for the two that did ($+0.10$, $+0.04$ to $+0.17$), leaving
out each judge ($+0.09$ to $+0.11$) and each sibling pair ($+0.08$ to $+0.12$),
and for designs whose placebo matching changed ($+0.11$) or left unchanged
($+0.09$); every interval excludes zero. Judgments of one prompt are correlated,
so we rely on these intervals rather than likelihood-ratio tests. A null that
counts all of GPT-4.1 nano's first-position preference as lapse
($\varepsilon=0.48$) can absorb its random-placebo shift, but fitted jointly with
position terms the lapse is small, and the matched placebo needs no such model.
\texttt{results/judge/conditioned.txt} gives the within-prompt test (matched
placebo, per judge: DeepSeek V3.2 $+0.12$, $0.02$ to $0.24$; Gemini 2.5 Flash-Lite
$-0.04$, $-0.15$ to $+0.08$; GPT-4.1 nano $+0.12$, $0.05$ to $0.20$; Claude Haiku
5.5 $-0.10$, $-0.25$ to $+0.06$; random placebo, mean $+0.01$, $-0.04$ to $+0.07$),
the prompt-level win-share correlations per pair, and the pairs-to-triads fit:
triad log-likelihood gain over $\lambda=1$ with pair-fitted nest parameters
$-15$, $-10$, $+14$ and $-22$ (DeepSeek, Gemini, nano, Haiku), against $+22$,
$+14$, $+18$ and $+6$ with two triad parameters fitted.

\begin{table*}[h]
\caption{Judge experiment per judge: similarity effect $\delta$ against the random
and the matched placebo, and the choice model's nest parameters $\lambda$ for
pairs and for three-answer menus, with 95\% intervals from resampling prompts.}
\label{tab:judges}
\centering\small
\setlength{\tabcolsep}{4pt}
\begin{tabular}{@{}lcccc@{}}
\toprule
 & GPT-4.1 nano & Claude Haiku 5.5 & Gemini 2.5 FL & DeepSeek V3.2 \\
\midrule
$\delta$, random $U$ & $+0.13$ ($+0.09$, $+0.19$) & $+0.08$ ($-0.04$, $+0.18$) & $+0.11$ ($+0.02$, $+0.19$) & $+0.07$ ($-0.02$, $+0.14$) \\
$\delta$, matched $U$ & $+0.14$ ($+0.08$, $+0.20$) & $+0.07$ ($-0.04$, $+0.18$) & $+0.08$ ($-0.03$, $+0.16$) & $+0.13$ ($+0.06$, $+0.21$) \\
designs with $\delta>0$ & 14 of 16 & 13 of 16 & 11 of 16 & 13 of 16 \\
$\lambda$, version pairs & 0.75 (0.64, 0.87) & 0.67 (0.56, 0.77) & 0.85 (0.75, 0.98) & 0.87 (0.75, 1.02) \\
$\lambda$, size pairs & 0.86 (0.77, 0.97) & 1.03 (0.84, 1.24) & 0.66 (0.56, 0.76) & 1.10 (0.91, 1.31) \\
$\lambda$, same-lab pairs & 0.84 (0.73, 0.97) & 0.99 (0.81, 1.21) & 0.81 (0.67, 0.99) & 0.97 (0.82, 1.17) \\
$\lambda$, version triads & 0.82 (0.73, 0.88) & 0.92 (0.82, 1.03) & 0.94 (0.86, 1.01) & 0.85 (0.77, 0.94) \\
$\lambda$, size triads & 0.88 (0.81, 0.96) & 1.21 (1.03, 1.40) & 0.83 (0.76, 0.93) & 0.81 (0.74, 0.89) \\
lapse $\varepsilon$ & 0.04 & 0.07 & 0.00 & 0.02 \\
\bottomrule
\end{tabular}
\end{table*}

The answers, prompts and every judgment are released with the code.

\section{Other Human Data}
\subsection{Human Ratings in PRISM}
\label{app:prism}

In the opening turn of a PRISM conversation \citep{kirk2024prism} a participant
writes a prompt, sees up to four answers from models drawn at random from 21
(sampled at temperature 1, about 50 words each) and rates each on a 1--100
slider; the top-rated answer is the chosen one. We use the 7,990 opening turns
with at least two answers. Relations follow our lineage table, with the six PRISM
models it lacks added; Claude Instant 1 and Command Nightly (a moving alias of
Command) are ambiguous. \Cref{tab:prism} summarizes
\texttt{results/prism/prism.txt}; intervals resample participants (200 draws;
100 for the nested logit).

\emph{Score correlation.} Each rating is standardized by its model's mean and
standard deviation; the mean product of two answers' residuals in a menu, by
relation, is the within-menu correlation. Shared prompt and rater effects give
unrelated pairs 0.14; the excess for relatives is the quantity of interest. It
survives dropping the 25 menus in which Claude 2.0 and 2.1 gave identical text
(version excess $+0.30$), removing each participant's mean rating ($+0.32$,
$+0.10$, $+0.16$) and dropping pairs with an ambiguous tag ($+0.64$ from Claude
2.0/2.1 alone, $+0.14$, $+0.25$). Unrelated pairs vary (standard deviation 0.11
across pairs; GPT-3.5 correlates with several unrelated models at about 0.4), so
some of the excess may be shared style rather than shared lineage.

\emph{Choices.} The decoy test is the Mantel--Haenszel log odds ratio of an
outsider $T$ over $S_1$ in menus that also hold $S_1$'s sibling against menus
with no relative of $S_1$, over 508 $(S_1,S_2,T)$ strata. Excluding the 482 menus
with a tied top rating (ties go to the first slot) gives $+0.22$ ($-0.02$ to
$+0.44$) for versions and $+0.50$ ($0.15$ to $0.78$) for Claude 2.0/2.1, which
gives $+0.40$ ($0.08$ to $0.69$) over all menus. If
ratings were multivariate normal with the observed means, spreads and
correlations, the predicted shift would be $+0.20$ for versions and $+0.09$ for
same-lab pairs, close to the observed $+0.22$ and $+0.05$, but $+0.04$ for sizes
against an observed $-0.05$. A nested logit with slot terms and a lapse (fitted
lapse 0.12) puts $\lambda$ below 1 for versions only. Pairwise multipliers,
treating the chosen answer as beating each other one, point the same way as
Arena's but are too uncertain to count as a replication.

\begin{table}[h]
\caption{PRISM human ratings by relation, with 95\% intervals from resampling
participants. Implied $k$ applies \cref{prop:thurstone} with $\rho_0=0.14$.}
\label{tab:prism}
\centering\small
\setlength{\tabcolsep}{3pt}
\begin{tabular}{@{}lccc@{}}
\toprule
 & Version & Size & Same lab \\
\midrule
model pairs & 2 & 6 & 6 \\
score correlation & 0.50 & 0.27 & 0.33 \\
excess over unrelated & $+0.36$ & $+0.13$ & $+0.19$ \\
 & (0.24, 0.50) & (0.06, 0.20) & (0.13, 0.26) \\
implied $k$ & 1.31 & 1.08 & 1.13 \\
decoy log odds ratio & $+0.17$ & $-0.03$ & $+0.04$ \\
 & ($-0.06$, 0.39) & ($-0.19$, 0.11) & ($-0.09$, 0.16) \\
nested-logit $\lambda$ & 0.59 & 1.06 & 0.90 \\
 & (0.33, 0.87) & (0.83, 1.32) & (0.68, 1.16) \\
pairwise $k$ & 1.30 & 1.18 & 1.15 \\
 & (0.77, 2.10) & (0.67, 1.85) & (0.96, 1.48) \\
\bottomrule
\end{tabular}
\end{table}

\subsection{Randomized Pairings in MultiPref}
\label{app:multipref}

MultiPref \citep{miranda2024hybrid} pairs answers from six models (GPT-4 Turbo,
GPT-3.5 Turbo, Llama 2 70B Chat, Llama 3 70B Instruct, Tulu 2 7B and 70B) in all
21 combinations, self-pairs included, about 500 times each by design; each of the
10,461 comparisons is rated by two crowdworkers and two experts on a five-point
scale. By our labels Tulu 2 7B/70B is a size pair and Llama 2/3 70B and GPT-3.5/4
Turbo are same-lab pairs; Tulu 2 is fine-tuned from Llama 2, so its two pairs with
Llama 2 get their own multiplier. Fitting \cref{eq:k} with a first-slot term to the
30,168 annotations of distinct-model pairs, ties dropped, gives $k$ of 1.47 (0.83
to 2.37) for size, 1.11 (0.93 to 1.31) for same lab and 1.18 (1.01 to 1.41) for the
shared base, with intervals from resampling prompts
(\texttt{results/multipref/multipref.txt}). Counting ties as half a win gives 1.32,
1.11 and 1.10. Crowdworkers alone give 1.81, 1.36 and 1.07, experts alone 1.14,
0.88 and 1.35. With one or two model pairs per relation the check is weak, but its
pairings are randomized, so it is free of the selection that Arena's pairings may
carry.

\subsection{Human Summarization Preferences}
\label{app:rm}

\paragraph{Data.} The comparisons released by \citet{stiennon2020learning} ask
trained labelers which of two summaries is better. We use the whole release: the
training split (batches 3 to 10, which the authors give as the training data of
their reward models, and batch 15; 92,842 comparisons after the exclusions below)
and the two validation splits (later batches, mainly used for evaluation; 84,625
comparisons). Almost all summarize a Reddit post; 2,284 validation comparisons
summarize a CNN or Daily Mail article. Summaries come from 81 sources:
human-written references and pretrained, supervised and PPO policies with
best-of-$n$ and decoding variants, from 1.3B to 13B parameters. From the source
names each pair is two samples of one policy (the on-policy case), two versions of
one model, two sizes of one family, two families, two human summaries, or a human
summary against a model's. Comparisons of two identical texts, or involving a text
listed under two sources on one post, are dropped. The labels, the model and the
primary outcome were committed before any outcome was computed.

\paragraph{Model.} A summary is one text on one post, and every comparison it
enters there shares its latent quality $u_s=\mu_{a(s)}+\tau_{g(s)}e_s$, with
$e_s\sim N(0,1)$ independent across summaries. Here $a(s)$ is its source and
$g(s)$ its spread group: the source itself if it has at least 2,000 summaries
(seven sources), otherwise all remaining human-written or all remaining model
sources. Labeler $w$ prefers summary $a$, shown first, to $b$ with probability
$\sigma\big(e^{\omega_w}k_{\mathrm{rel}(a,b)}(u_a-u_b)+c\big)$, as in \cref{eq:k}.
The $e_s$ are integrated out jointly over all comparisons on a post with 128 fixed
normal draws (simulated maximum likelihood), so a source whose samples simply vary
more cannot pass for a larger $k$. Location and scale are fixed by
$\mu_{\mathrm{ref}}=0$ for the human references and $k=1$ for a human summary
against a model's; the labeler scales $\omega_w$ are centered and lightly shrunk.

\paragraph{Recovery.} With outcomes simulated from the model, with known
parameters, on the real comparisons, planted same-policy $\log k$ of 0, 0.26 and 0.26 (one data set each,
with different spreads and other multipliers) came back as $-0.05$, 0.34 and
0.27. Errors of this size are about half the estimated effect
($\log 1.17=0.16$); with one data set per scenario they mix sampling noise with
the bias of the simulated likelihood, and they are a further reason for caution.

\paragraph{Results.} On 177,467 comparisons in 21,462 posts, two samples of one
policy give $k=1.17$ (1.08 to 1.24, resampling posts); versions give 1.05 (0.98 to
1.12), sizes 0.89 (0.81 to 0.98) and two families 1.04 (0.97 to 1.14). Three checks
do not support reading this as the similarity effect. Fitted separately, the
training split gives a same-policy $k$ of 1.02 (0.94 to 1.08) and the validation
splits 1.08 (0.93 to 1.18). Among model pairs, $\log k$ falls with the MiniLM
cosine of the two summaries ($-0.076$ per standard deviation, $-0.114$ to $-0.049$),
the opposite of the Arena pattern. Labelers' stated confidence is lower on
same-policy pairs than the quality gap implied by the post's other comparisons
predicts. In addition, with 1,024 instead of 128 simulation draws the same-policy
$k$ is 1.16 but the two-family $k$ rises to 1.15, so the contrast with other model
pairs depends on the simulation. The data are from 2020 summarizers labeled by
trained contractors, and a human-written summary differs from a model's in more
than relatedness. We read the test as inconclusive for the on-policy case
(\texttt{results/rm/}).

\end{document}